\documentclass[sn-mathphys-num]{sn-jnl}

\usepackage{graphicx}%
\usepackage{natbib}
\usepackage{amsmath,amssymb,amsfonts}%
\usepackage{amsthm}%
\usepackage{mathrsfs}%
\usepackage[title]{appendix}%
\usepackage{xcolor}%
\usepackage{textcomp}%
\usepackage{manyfoot}%
\usepackage{algorithm}%
\usepackage{algorithmicx}%
\usepackage{algpseudocode}%
\usepackage{listings}%
\theoremstyle{thmstyleone}%
\newtheorem{theorem}{Theorem}
\usepackage{pdfpages}
\usepackage{geometry}
\usepackage{multirow}%
\usepackage{booktabs}%
\usepackage{subcaption}

\theoremstyle{thmstyletwo}%

\theoremstyle{thmstylethree}%
\newtheorem{definition}{Definition}%

\newcommand{\TheName}{\textbf{\tt Spofe}}
\newcommand{\revise}[1]{#1}

\title{Interpretable Feature Interaction via Statistical Self-supervised Learning on Tabular Data}
\date{}
\begin{document}

\author[1]{\fnm{Xiaochen} \sur{Zhang}}\email{202012079@mail.sdu.edu.cn}
\equalcont{}
\author*[2]{\fnm{Haoyi} \sur{Xiong}}\email{haoyi.xiong.fr@ieee.org}
\equalcont{These authors contributed equally to this work. This work was conducted through an independent collaboration between the two authors and bears no connection to the second author's affiliation.}

\affil[1]{\orgdiv{Research Center for Mathematics and Interdisciplinary Sciences}, \orgname{Shandong University}, \postcode{266237}, \orgaddress{\city{Qingdao}, \country{China}}}
\affil[2]{\orgname{Independent Researcher}, \city{Haidian District}, \postcode{100085}, \state{Beijing}, \country{China}}

\abstract{
In high-dimensional and high-stakes contexts, ensuring both rigorous statistical guarantees and interpretability in feature extraction from complex tabular data remains a formidable challenge. Traditional methods such as Principal Component Analysis (PCA) reduce dimensionality and identify key features that explain the most variance, but are constrained by their reliance on linear assumptions. In contrast, neural networks offer assumption-free feature extraction through self-supervised learning techniques such as autoencoders, though their interpretability remains a challenge in fields requiring transparency. To address this gap, this paper introduces \TheName{}, a novel self-supervised machine learning pipeline that marries the power of kernel principal components for capturing nonlinear dependencies with a sparse and principled polynomial representation to achieve clear interpretability with statistical rigor. Underpinning our approach is a robust theoretical framework that delivers precise error bounds and rigorous false discovery rate (FDR) control via a multi-objective knockoff selection procedure; it effectively bridges the gap between data-driven complexity and statistical reliability via three stages: (1) generating self-supervised signals using kernel principal components to model complex patterns, (2) distilling these signals into sparse polynomial functions for improved interpretability, and (3) applying a multi-objective knockoff selection procedure with significance testing to rigorously identify important features. Extensive experiments on diverse real-world datasets demonstrate the effectiveness of \TheName{}, consistently surpassing KPCA, SKPCA, and other methods in feature selection for regression and classification tasks. Visualization and case studies highlight its ability to uncover key insights, enhancing interpretability and practical utility.}

\maketitle

\section{Introduction}
In the burgeoning field of data science, unsupervised feature extraction and selection have become foundational to exploratory data analysis and advanced analytics~\cite{ICA,linearPCA,van2008visualizing}. These techniques not only reduce dimensionality and simplify complex data into more interpretable formats but also significantly enhance the performance of predictive models across various domains. For example, in healthcare, unsupervised feature extraction improves diagnostic accuracy by identifying key biomarkers from extensive medical datasets. In finance, it helps detect anomalies such as fraudulent transactions by isolating essential features from large-scale data. In retail, it supports customer segmentation and personalization by analyzing purchase histories, allowing businesses to tailor their services. These examples highlight the transformative potential of feature extraction in uncovering actionable insights from complex datasets across industries.

\revise{Traditional methods like kernel Principal Component Analysis (KPCA)\cite{KPCA} was introduced, applying a nonlinear mapping to a higher-dimensional space before performing PCA\cite{linearPCA}. This allows it to capture complex relationships that traditional PCA overlooks, making it well-suited for analyzing nonlinear behaviors and feature interactions in data. However, challenges remain, such as the need for precise kernel selection, vulnerability to overfitting, and significant computational demands that hinder scalability. Additionally, the transformation in Kernel PCA makes component interpretation more difficult, as the underlying features become less transparent. These challenges emphasize the need for more adaptable, assumption-free unsupervised learning methods that can better handle modern data complexities.}

\revise{One critical aspect of modern data analysis that demands such interpretability is the modeling of feature interactions. Feature interactions are a critical aspect of understanding complex datasets, as they capture how multiple features jointly influence outcomes in nonlinear or non-additive ways. Research on feature interactions has proposed diverse approaches, including statistical methods, information-theoretic measures, and machine learning techniques. Statistical approaches, such as functional ANOVA decomposition~\cite{DiscoveringAdditive} and rule-based ensemble models~\cite{RuleEnsemble}, focus on quantifying interaction strength but struggle to handle high-dimensional, nonlinear, and high-order interactions. Information-theoretic methods, like mutual information-based feature selection~\cite{FeatureSelectionMI,FeatureSelectionConditionalMI}, detect interaction relevance but often lack explicit mathematical representations and rely heavily on supervised learning. Machine learning techniques, such as  neural networks and multifactor dimensionality reduction (MDR)~\cite{MLGeneGeneInteractions,InterpretationReview}, excel at identifying nonlinear interactions but are often computationally expensive and suffer from a lack of interpretability due to their "black-box" nature.} 

\begin{table}[ht]
  \centering
  \footnotesize
  \caption{Comparison of Feature Learning Methods. The Fisher's ratio refers to the ratio of between-class variance and within-class variance.}
  \label{tab:comparison_methods}
 \begin{tabular}{llll}
 \toprule
  \textbf{Method} & Feature Interactions & Interpretation & \textbf{(Un)supervision} \\\midrule 
{PCA} 
  & {N/A} 
  & {N/A} 
  & {Maximize variance} 
\\ 

{Sparse PCA} 
  & {N/A} 
  & {Feature Attribution} 
  & {Maximize variance} 
\\ 

{LDA} 
  & {N/A} 
  & {N/A} 
  & {Maximize the Fisher's ratio} 
\\

{Sparse LDA} 
  & {N/A} 
  & {Feature Attribution} 
  & {Maximize the Fisher's ratio} 
\\ 

{KPCA} 
  & {Kernel function} 
  & {N/A} 
  & {Maximize variance} 
\\

{Sparse KPCA} 
  & {Kernel function} 
  & {Sample Influence} 
  & {Maximize variance} 
\\

{KDA} 
  & {Kernel function} 
  & {N/A} 
  & {Maximize the Fisher's ratio} 
\\

{Sparse KDA} 
  & {Kernel function} 
  & {Sample Influence} 
  & {Maximize the Fisher's ratio} 
\\

SDA with Poly. features 
  & {Polynomial features} 
  & {Feature Attribution} 
  & {Maximize the Fisher's ratio} 
\\

Our Work (\TheName{}) 
  & {Polynomial features} 
  & {Feature Attribution} 
  & {Maximize variance} 
\\
\bottomrule
\end{tabular}
\end{table}

\revise{To address these limitations, we introduce \TheName{}, a novel self-supervised machine learning pipeline designed for interpretable and statistically significant feature interaction detection. \TheName{} leverages the ability of KPCA to capture complex, nonlinear structures within data while addressing its limitations in interpretability by integrating a Knockoff filter framework. By combining these approaches, \TheName{} enables the detection of significant feature interactions in tabular data with both computational efficiency and statistical rigor, specifically tailored for scenarios where label information is unavailable.  \TheName{} achieves this through three key components: \emph{self-supervised signal generation}, \emph{multi-objective Knockoff selection procedure}, and \emph{significance testing}. Experiments on real-world datasets validated the effectiveness of this pipeline. Our contributions are as follows:}

\begin{itemize} \item \revise{We propose to approximate kernel principal components (KPCs) using sparse polynomial functions, enabling self-supervised detection of significant feature interactions. This approach reframes the challenge of interpretability as a selective inference problem within a sparse polynomial space. By doing so, it bridges the gap between the complexity of kernel-based methods and the need for interpretable feature interaction modeling, allowing for rigorous hypothesis testing on polynomial coefficients. This enables the direct identification of both individual features and their significant interactions that contribute to the underlying data structure. Table~\ref{tab:comparison_methods} presents a comparison of various feature learning methods, highlighting their feature interactions, interpretations, and the ability to capture nonlinearity in kernel space, with a particular focus on how our work enhances these aspects through feature attribution together with polynomial expansions in unsupervised settings.}

\item \revise{\TheName{} integrates statistical self-supervised signal generation via KPCA, Knockoff methods for multi-objective feature selection, and hypothesis testing into a unified framework tailored for feature interaction analysis. First, it projects the data onto the top kernel principal components to capture complex interaction patterns in an unsupervised manner. Next, it constructs synthetic knockoff statistics, weighting them by the proportion of variance explained by each principal component, to identify not only individual significant features but also their interactions across multiple dimensions. Finally, it estimates \(p\)-values for polynomial features within a hypothesis testing framework, achieving statistically valid and interpretable detection of feature interactions. To our knowledge, this is the first approach to combine self-supervised signals and Knockoff methods within a sparse polynomial space to detect and interpret significant feature interactions in tabular data.}

\item \revise{We validate \TheName{} through extensive experiments on diverse real-world datasets, highlighting its versatility across domains such as healthcare, meteorology, and high-energy physics. \TheName{} consistently outperforms established supervised and unsupervised feature selection methods, including KPCA, Sparse Kernel Principal Component Analysis (SKPCA), and other baselines, in both regression and classification tasks. Visualization studies confirm that it accurately reconstructs the structures captured by KPCA, while its feature selection demonstrates strong interpretability by uncovering significant insights in real-world case studies.}

\end{itemize}

The remainder of this paper is organized as follows. Section 2 presents the necessary background on kernel methods, feature selection, and statistical inference. Section 3 introduces our proposed \TheName{} framework, detailing its methodology for generating interpretable features through polynomial approximations and statistical selection. Section 4 provides theoretical analysis of the framework's statistical guarantees and computational complexity. Section 5 presents extensive experimental results on both synthetic and real-world datasets, demonstrating the effectiveness of \TheName{} in feature interpretation and prediction tasks. Finally, Section 6 concludes the paper with discussions and future research directions.

\section{Related Works: Backgrounds and Preliminaries}
This section establishes the theoretical foundations of our framework, covering kernel-based feature extraction methods, statistical approaches to feature selection, and polynomial approximation techniques.

\subsection{Kernel Principal Component Analysis (KPCA)}
KPCA~\cite{KPCA} extends PCA nonlinearly by mapping data $X = [\boldsymbol{x}_1,\ldots,\boldsymbol{x}_n]^\top \in \mathbb{R}^{n\times p}$ via a kernel $k(\boldsymbol{x}_i,\boldsymbol{x}_{i'})$. Define the kernel matrix
$
K_{i,i'} = k(\boldsymbol{x}_i, \boldsymbol{x}_{i'}), \quad i,i'=1,\dots,n.
$
Centering is performed as
\begin{equation}\label{eq:center_K}
K_{\text{centered}} = K - \mathbf{1}_nK - K\mathbf{1}_n + \mathbf{1}_nK\mathbf{1}_n,\quad \mathbf{1}_n = \frac{1}{n}\mathbf{1}\mathbf{1}^\top.
\end{equation}
Eigen-decomposition is then defined by
\begin{equation}
n\lambda\,\boldsymbol{\alpha} = K_{\text{centered}}\boldsymbol{\alpha}, \quad \text{with } 1=\lambda_i(\boldsymbol{\alpha}_i^\top\boldsymbol{\alpha}_i),
\end{equation}
and the $j$th kernel principal component (kPC) is
\begin{equation}
\text{(kPC)}(\boldsymbol{x})=\sum_{i=1}^{n} k(\boldsymbol{x},\boldsymbol{x}_i)\boldsymbol{\alpha}_i.
\end{equation}

\subsection{Unsupervised and Self-supervised Feature Extraction}
Feature extraction reduces data dimensionality while preserving essential characteristics, and both unsupervised and self-supervised methods achieve this without using sample labels. \emph{Unsupervised techniques} such as Principal Component Analysis (PCA)~\cite{linearPCA} and Independent Component Analysis (ICA)~\cite{ICA} transform data into uncorrelated or statistically independent components that capture maximum variance or latent features; PCA has variants including Sparse PCA, Kernel PCA, Sparse Kernel PCA, and functional PCA, while visualization-based methods like t-Distributed Stochastic Neighbor Embedding (t-SNE)~\cite{van2008visualizing} and Self-Organizing Maps (SOM)~\cite{kohonen2013essentials} map high-dimensional data into two or three dimensions preserving local relationships. 

In contrast, \emph{self-supervised approaches} such as the Deep Boltzmann Machine (DBM)~\cite{salakhutdinov2009deep}, Autoencoders (AEs)~\cite{kingma2013auto} along with their variational extensions~\cite{kingma2019introduction} and masked variants~\cite{he2022masked}, and Autoregressive Learning methods~\cite{radford2018improving} are designed to learn feature representations by reconstructing inputs or predicting future sequence elements. While unsupervised methods like PCA and ICA offer clearer interpretability through quantitative analysis of variance or independence, self-supervised methods, despite their power in capturing complex, nonlinear relationships, often possess less transparent statistical significance due to their emphasis on learning connectivity over explicit statistical relationships~\cite{kingma2013auto,salakhutdinov2009deep}.

\subsection{Knockoff Methods}

Knockoff methods are broadly classified into fixed-X and model-X approaches~\cite{fixedKnockoff,modelXKnockoff}. The fixed-X method generates synthetic variables under linear regression assumptions, whereas the model-X framework accommodates more complex dependency structures by constructing a knockoff matrix $\widetilde{X}$ that mimics the distribution of the original variables $X$ while satisfying two key properties: \emph{Exchangeability}, meaning that $(X, \widetilde{X})_{\mathrm{swap}(S)} \overset{d}{=} (X, \widetilde{X})$ for any subset $S$ when swapping $X_i$ with $\widetilde{X}_i$, and \emph{Independence}, ensuring that $\widetilde{X}$ remains independent of the response vector $\mathbf{y}$ conditional on $X$. Variable significance is assessed using the knockoff statistic $W_i = w_i([X, \widetilde{X}], \mathbf{y})$, where the sign of $w_i$ is inverted when $X_i$ and $\widetilde{X}_i$ are swapped, i.e.,
\begin{equation}\label{eq:knockoff_statistics}
w_i([X, \widetilde{X}]_{\mathrm{swap}(S)}, \mathbf{y}) = 
\begin{cases}
w_i([X, \widetilde{X}], \mathbf{y}), & i \notin S, \\
-w_i([X, \widetilde{X}], \mathbf{y}), & i \in S,
\end{cases}
\end{equation}\label{eq:knockoff threshold}
and a threshold $\tau$ is determined by 
\begin{equation}
    \tau = \min\left\{ t > 0 : \frac{1 + \#\{i : W_i \leq -t\}}{\#\{i : W_i \geq t\}} \leq q \right\},
\end{equation}
which controls the false discovery rate (FDR) $q$ by selecting variables with $W_i \geq \tau$.

\subsection{Discussions: Interactions and Interpretations}

Feature interactions describe how variables work together to influence predictive performance and model interpretability. Explicit methods, such as polynomial feature expansion~\cite{de2019approximate}, functional ANOVA decomposition~\cite{DiscoveringAdditive}, and rule-based ensemble learning~\cite{RuleEnsemble}, directly model nonlinear relationships by constructing new variables that capture interaction effects and decompose variance, thereby providing clear insights into the roles of various feature combinations. In contrast, implicit approaches based on kernel methods~\cite{hofmann2008kernel}, deep neural networks~\cite{lecun2015deep}, Boltzmann machines~\cite{salakhutdinov2009deep}, autoencoders~\cite{kingma2013auto}, and multi-factor dimensionality reduction~\cite{MLGeneGeneInteractions,InterpretationReview} learn complex, often nonlinear inter-dependencies without explicitly defining interaction forms, which can limit interpretability and statistical validation, particularly in high-dimensional settings.

Interpretable machine learning seeks to clarify model decision-making by providing insights through feature attribution and sample influence. Self-interpretable models inherently assign importance to features—often via sparse representations or intrinsic weighting mechanisms—and can reveal the influence of individual training samples through kernel-based analyses or influence functions~\cite{debruyne2008model,de2012robustness,wojnowicz2016influence,zhang2022rethinking}. Global interpretations offer an overarching view by summarizing how features collectively impact predictions~\cite{li2022interpretable}, while local interpretations delve into instance-specific effects~\cite{li2020design}, together enabling robust and transparent decision-making in high-stakes applications.

Our work, \TheName{}, uniquely integrates explicit polynomial feature expansion with implicit kernel PCA to capture nonlinear interactions, representing them as structured, sparse polynomial forms and thus eliminating the need for manual specification of interaction terms. By leveraging weighted knockoff selection for statistically rigorous identification of significant interactions and incorporating hypothesis testing for coefficient validation, our method provides transparent insights without sacrificing the flexibility of unsupervised learning. This global interpretative framework, which complements detailed local analyses, offers a comprehensive understanding of the main drivers of model decisions and establishes our approach as a novel contribution in the modeling of complex, high-dimensional data (see Table~\ref{tab:comparison_methods}).

\section{Models and Formulations}

This section presents the formulation of the interpretability problem in Kernel Principal Component Analysis and proposes a sparse polynomial framework to enhance interpretability while retaining the capability of feature transformation.

\subsection{Interpretability Problem of Kernel PCA}

Kernel PCA identifies eigenvectors \(\boldsymbol{\alpha}_i\) in the Reproducing Kernel Hilbert Space (RKHS) \(\mathcal{H}_k\), but these are difficult to interpret due to their representation as complex, high-dimensional transformations of the input data. This makes it challenging to directly relate the components to the original features, limiting their utility for understanding the underlying data structure. Moreover, though KPCA captures the intrinsic structure of the data, its projection into orthogonal dimensions restricts each kernel principal component (KPC) to representing only a single aspect of the structure. These limitations underscore the need for a feature selection method that provides a broader, more comprehensive representation by identifying fundamental features that convey information from multiple directions within the data.

\subsection{Sparse Polynomial Approximation for Interpretability}

Building on the limitations of KPCA discussed earlier, we propose to approximate the kernel principal components using a finite linear combination of polynomial basis functions within a sparse polynomial function space.

\begin{definition}[Sparse Polynomial Function Space]
Let \(\mathcal{S}_\psi = \{ \psi_d \}_{d=1}^{d_{max}}\) be a set of polynomial basis functions of \(\boldsymbol{x}\). The sparse polynomial function space \( \mathcal{Q}_r \) is defined as:
\begin{equation}\mathcal{Q}_r = \left\{ f \,\Bigg|\, f(\boldsymbol{x}) = \sum_{d \in D} \beta_d \psi_d(\boldsymbol{x}),\; D \subseteq \{1, \dots, d_{max}\},\; |D| \leq r,\; \beta_d \in \mathbb{R} \right\}.
\end{equation}
\end{definition}

In this formulation, \(r\) limits the number of non-zero polynomial terms, ensuring a sparse and interpretable representation. Unlike sparse linear PCA, which captures only linear structures, polynomial functions allow us to model nonlinear relationships while remaining interpretable. By expressing the eigenvectors through these sparse polynomial functions, we gain deeper insights into how individual features contribute to the structure of the data across multiple dimensions.

\subsection{Representation of KPCA with Polynomial Functions}

To address the interpretability challenges in KPCA discussed earlier, we propose representing each kernel principal component \(\boldsymbol{z}_j\) using a linear combination of polynomial terms from the sparse polynomial function space \( \mathcal{Q}_r \). 

Let $\boldsymbol{x}_1,\ \boldsymbol{x}_2,\ \boldsymbol{x}_3,\dots\boldsymbol{x}_n$ be the $n$ random vectors of $p$-dimension. Let \(\boldsymbol{\alpha}_j = (\alpha_{j,1}, \alpha_{j,2}, \dots, \alpha_{j,n})^\top\) represent the eigenvector associated with the \(j\)-th kernel principal component. For each data point \(i\), the score of the \(j\)-th kernel principal component is computed as:
\begin{equation}
z_{i,j} = \sum_{i'=1}^{n} k(\boldsymbol{x}_{i}, \boldsymbol{x}_{i'}) \alpha_{j,i'}.
\end{equation}
%
We consider that the discrepancy between each computed kernel principal component score and its expected value is due to statistical error. For each data point $i$ and component $j$, define the error as
\begin{equation}\label{eq:error-define}
   z_{i,j}  =  \mathbb{E}\left[z_{i,j}\right]+\varepsilon_{i,j}\ \text{and}\ \varepsilon_{i,j} \sim \mathcal{N}(0,\sigma(p,n)),
\end{equation}
where the variance $\sigma(p,n)$ is modeled as a function depending on dimension $p$ and sample size $n$. This setup implies that the variability in the computed scores around their expected values is modeled by Gaussian noise with variance that may adapt according to $d$ and $n$. Such a characterization facilitates the application of concentration inequalities and statistical error bounds when analyzing the gap between the observed scores and their expectations. Hereby, we define the sparse polynomial representation of the \(j\)-th kernel principal component as follows.

\begin{definition}[Sparse Polynomial Representation of Kernel PCA]\label{def:sparse_representation}
A sparse polynomial representation of the Kernel Principal Components Analysis is a method that approximates each kernel PC $\boldsymbol{z}_j$ through a limited subset of polynomial basis functions to enable interpretability. For the $j$-th kernel principal component, this representation is obtained by solving the optimization problem:
\begin{equation}
\boldsymbol{\beta}^{*(j)}\gets\min_{\boldsymbol{\beta}^{(j)} \in \mathbb{R}^{d_{max}}} \frac{1}{n} \sum_{i=1}^n \left( \mathbb{E}[z_{i,j}] - \sum_{d=1}^{d_{max}} \beta_d^{(j)}\, \psi_d(\boldsymbol{x}_i) \right)^2 \quad \text{subject to} \quad \| \boldsymbol{\beta}^{(j)} \|_0 \leq r. \label{eq:optimization}
\end{equation}
where: 
\begin{itemize}
    \item $z_{i,j}$ is the score of the $j$-th kernel principal component for the $i$-th data point and $\mathbb{E}[\cdot]$ refers to its expectation value subject to the random vectors $\boldsymbol{x}_1,\ \boldsymbol{x}_2,\ \boldsymbol{x}_3,\dots\boldsymbol{x}_n$; 
    \item $\psi_d(\boldsymbol{x}_i)$ represents the $d$-th polynomial basis function evaluated at data point $\boldsymbol{x}_i$;
    \item $\boldsymbol{\beta}^{(j)} = (\beta_1^{(j)}, \dots, \beta_{d_{max}}^{(j)})^\top$ is the vector of polynomial coefficients and $\boldsymbol{\beta}^{*(j)}$ refers to the optimal solution of the problem~\ref{eq:optimization};
    \item $\| \boldsymbol{\beta}^{(j)} \|_0 \leq r$ is a sparsity constraint limiting the number of non-zero coefficients to at most $r$; and (4) $d_{max}$ is the total number of available polynomial basis functions.
\end{itemize}
This representation transforms the complex, high-dimensional kernel principal components into interpretable polynomial functions while preserving their ability to capture nonlinear data structures.

\end{definition}


Rather than focusing on the exact values of the coefficients \( \beta_d^{*(j)} \), the goal is to identify the selected polynomial features, which are those with non-zero coefficients. These features represent the most important contributors to the kernel principal components. To formalize this, we define the significant polynomial terms, which highlight the features that play the most crucial role in describing the kernel principal components.\\

\begin{definition}[Significant Polynomial Terms]
The set of significant polynomial terms \(A^*_j\) for the kernel principal component \(\boldsymbol{z}_j\) is defined as the set of indices \(d\) corresponding to non-zero coefficients in the solution of the optimization problem \eqref{eq:optimization}:
\begin{equation}
A^*_j = \left\{ d \in \{1, 2, \dots, d_{max}\} \,\Bigg|\, \beta_d^{*(j)} \neq 0 \text{ in the solution of \eqref{eq:optimization}} \right\}. \label{eq:significantTerms}
\end{equation}
\end{definition}
In this context, a non-zero coefficient \( \beta_d^{*(j)} \) indicates that the corresponding polynomial term \( \psi_d \) is significant in the representation of the \(j\)-th kernel principal component, providing insight into the structure of the data.

\subsection{Feature Identification with Sparse Polynomial Functions} \label{subsection:hypothesisSection}
Having approximated the kernel principal components with polynomial terms, we frame the search for significant features as a statistical inference problem. To this end, we establish a hypothesis testing framework to evaluate the significance of each polynomial feature as follows.

\begin{definition}[Hypotheses for Feature Selection]\label{def:hypothesis}
For the kernel principal component \(\boldsymbol{z}_j\) and each polynomial feature \( \psi_d \), we define the following hypotheses:
\[
H^{(j,d)}_0: \beta_d^{*(j)} = 0 \quad (\text{Null Hypothesis}), \quad H^{(j,d)}_1: \beta_d^{*(j)} \neq 0 \quad (\text{Alternative Hypothesis})
\]
where \(\beta_d^{*(j)}\) is defined in \eqref{eq:optimization}.
\end{definition}

This hypothesis framework enables us to assess whether each polynomial feature \( \psi_d \) significantly contributes to the kernel principal component by checking if its coefficient is non-zero. The subsequent section outlines the method for testing these hypotheses and validating feature significance.

\section{Algorithm and Analysis: Self-supervised Sparse Polynomial Feature Extraction}
\label{sec:section3}
This section presents the overall framework and core algorithms of \TheName{}.

\subsection{Overall Framework Design}
In this section, we describe the comprehensive design of the \TheName{} algorithm (Algorithm~\ref{alg:Spofe}). The algorithm accepts the training data matrix \(X\) as input and outputs the selection result of the polynomial features. The method comprises four principal steps: \emph{statistical self-supervised signal generation with KPCA, weighted Knockoff for multiple objectives, \(p\)-values estimating for polynomial features, and feature selection.} Each step is elaborated upon as follows.

\subsubsection{Statistical Self-supervised Signal Generation with KPCA}

Given the input data matrix \(X\), \TheName{} begins by generating a set of statistical signals using Kernel Principal Component Analysis (KPCA), as described in Line 1 of Algorithm~\ref{alg:Spofe}. This process involves computing the kernel matrix of \(X\), obtaining the vector \(\boldsymbol{\Lambda}_m\) that contains the top \(m\) scaled eigenvalues, and projecting \(X\) onto the corresponding \(m\) eigenvectors to form the signal matrix \(M\). These KPCA-based signals provide a self-supervised representation of the data, which is crucial for assessing the importance of polynomial features in later steps.

\subsubsection{Weighted Knockoff for Multiple Objectives}

After generating the statistical signals, \TheName{} constructs a polynomial feature matrix \(\Psi\) with \(d_{max}\) feature columns based on the input data \(X\), as described in Line 2 of Algorithm~\ref{alg:Spofe}. It then applies the weighted knockoff procedure (Line 3), which creates knockoff features \(\tilde{\Psi}\) and computes the knockoff statistics for each feature in \(\Psi\). These statistics, weighted by the scaled eigenvalues in \(\boldsymbol{\Lambda}_m\) from the former step, result in a significance score vector \(\boldsymbol{s}\), which ranks the importance of each polynomial feature.

\subsubsection{\(p\)-values Estimating for Polynomial Features}

With the significance score vector \(\boldsymbol{s}\) calculated, \TheName{} estimates the \(p\)-values for each polynomial feature (Line 4 of Algorithm~\ref{alg:Spofe}). This is done by fitting a distribution family \(\mathcal{P}(\mu, \sigma ^2)\) to the scores from \(\boldsymbol{s}\) and then calculating the corresponding \(p\)-values for each feature. These \(p\)-values provide a statistical measure of the relevance of each polynomial feature, guiding the feature selection process.

\subsubsection{Feature Selection}
As mentioned earlier, the previous step of \TheName{} provides a significance ranking of the polynomial features based on the estimated \(p\)-values. To implement \(\texttt{selectFeatures}(\cdot)\) in Line 5 of Algorithm~\ref{alg:Spofe}, \TheName{} offers three strategies for selecting the most relevant features:

\begin{itemize}
    \item \emph{Selection by \(p\)-value threshold}: Features with \(p\)-values below a specified threshold are selected. To control the false discovery rate, adjustments such as the Benjamini-Hochberg procedure~\cite{benjamini1995controlling} may be employed.

    \item \emph{Fixed-length Selection}: A predetermined number of features with the lowest \(p\)-values are selected, which is particularly useful when there is a restriction on the number of features to be included.
    
    \item \emph{Varying-length Selection}: When the optimal number of features is unknown, cross-validation or a validation dataset is used to evaluate model performance across different feature sets, facilitating the identification of the most appropriate subset.
\end{itemize}
In our experiments (Section~\ref{section:experiments}), fixed-length selection was primarily used for evaluating prediction accuracy, as it ensures a consistent number of features across methods. In contrast, percentile-based selection was employed for interpretability analysis, as it allows for a more flexible and significance-driven selection of features.

In the following part of this section, we present the details of the three core algorithms that form the foundation of our framework: \emph{statistical self-supervised signal generation with KPCA}, \emph{weighted knockoff for multiple objectives}, and \emph{p-values estimation for polynomial features}.

\begin{algorithm}
\caption{\texttt{Spofe}: Self-supervised Sparse Polynomial Feature Extraction}
\label{alg:Spofe}
\begin{algorithmic}[1]
\Statex \textbf{Input:} Training data matrix \( X = [\boldsymbol{x}_1, \ldots, \boldsymbol{x}_{n}]^\top \in \mathbb{R}^{n \times p}\).
\Statex \textbf{Parameters:} Distribution family \(\mathcal{P}(\mu, \sigma ^2)\) for the significance scores \(s_d\), with \(\mu\) and \(\sigma\) estimated from score vector \(\boldsymbol{s} \in \mathbb{R}^{d_{max}}\).
\Statex \textbf{Output:} $p$-value vector \(\hat{\boldsymbol{p}} \in \mathbb{R}^{d_{max}}\) for features in the polynomial feature matrix \(\Psi  \in \mathbb{R}^{n \times d_{max}}\). Set of selected column indices \( \hat{A}_{\TheName} \) from matrix \(\Psi\).

\State \(M, \boldsymbol{\Lambda}_m \gets \texttt{S4Gen}(X)\) \Comment{Generate signals via self-supervised learning}
\State \(\Psi \gets \texttt{constructPolynomialMatrix}(X)\) \Comment{Construct polynomial matrix with \(d_{max}\) features}

\State \(\boldsymbol{s} \gets \texttt{WEKO}(\Psi, M, \boldsymbol{\Lambda}_m)\)  \Comment{Compute weighted knockoff scores as significance metric}

\State \(\hat{\boldsymbol{p}} \gets \texttt{estPVal}(\boldsymbol{s}, \mathcal{P}(\mu, \sigma))\) \Comment{Estimate $p$-values using distribution \(\mathcal{P}(\mu, \sigma ^2)\)}

\State \(\hat{A}_{\TheName} \gets \texttt{selectFeature}(\hat{\boldsymbol{p}})\) \Comment{Select top polynomial features by $p$-values}

\State \Return \( \hat{\boldsymbol{p}} \), \( \hat{A}_{\TheName} \) \Comment{Return $p$-values and selected polynomial features.}
\end{algorithmic}
\end{algorithm}

\subsection{Statistical Self-supervised Signal Generation with KPCA}
This section explains the process of \emph{statistical self-supervised signal generation} through Kernel Principal Component Analysis (KPCA).

\subsubsection{Kernel Matrix Computation}
The algorithm begins by calculating and centering the kernel matrix \( K \in \mathbb{R}^{n \times n} \), with entries \( K_{i,i'} = k(\boldsymbol{x}_i, \boldsymbol{x}_{i'}) \). The kernel function \( k(\cdot, \cdot) \) is chosen based on the data characteristics, and the centering process produces the matrix \( K_{\text{centered}} \), ensuring zero mean in the feature space (lines 1 and 2 in Algorithm~\ref{alg:S4Gen}, as shown in \eqref{eq:center_K}). Alternatively, random Fourier features (RFF)~\cite{randomfourierfeatures} can approximate the kernel matrix, offering an efficient and widely used approach for scaling up kernel methods.

\subsubsection{Eigen-decomposition and Dimensionality Reduction}
Eigen-decomposition is applied to \( K_{\text{centered}} \) to yield the top \( m \) scaled eigenvalues \( \lambda_1 \geq \lambda_2 \geq \cdots \geq \lambda_m \) and the corresponding projection matrix \( V_m = [\boldsymbol{\alpha}_1, \boldsymbol{\alpha}_2, \ldots, \boldsymbol{\alpha}_m] \in \mathbb{R}^{n \times m} \) (line 3 in Algorithm~\ref{alg:S4Gen}), solving the eigenvalue problem in KPCA estimates. 
These eigenvalues form the vector \( \boldsymbol{\Lambda}_m = [\lambda_1, \ldots, \lambda_m]^\top \) and will be used in later stages. The parameter \(m\) is a predetermined value, typically chosen based on the desired level of dimensionality reduction and the significance of the eigenvalues.

\subsubsection{Signal Projection}
With the top \( m \) eigenvalues and eigenvectors obtained, the next step is to project the data onto the top \( m \) components by computing the product of \( K_{\text{centered}} \) and \( V_m \), forming the signal matrix \( M \in \mathbb{R}^{n \times m}\):
\begin{equation}
M = K_{\text{centered}} \cdot V_m 
\end{equation}
Each column of \( M \) corresponds to one of the first \( m \) principal components \( \boldsymbol{z}_j \), where \( z_{i,j} \) represents the score of the \( j \)-th component for the \( i \)-th data point. Unlike standard KPCA, which projects data onto all \( n \) eigenvectors, this method focuses on the top \( m \) components, reducing dimensionality while preserving the most important structures (for detailed mathematical derivations, see \cite{KPCA}).

Finally, the projected signal matrix \( M \) and the scaled eigenvalue vector \( \boldsymbol{\Lambda}_m \) are returned. These outputs serve as the foundation for assessing feature significance in subsequent steps.

\begin{algorithm}
\caption{\texttt{S4Gen}: Statistical Self-supervised Signal Generation with KPCA}
\label{alg:S4Gen}
\begin{algorithmic}[1]
\Statex \textbf{Input:} Data matrix \( X = [\boldsymbol{x}_1, \ldots, \boldsymbol{x}_n]^\top \in \mathbb{R}^{n \times p} \). Kernel function \(k(\cdot, \cdot)\).
\Statex \textbf{Parameters:} Number \( m \) of top scaled eigenvalues and eigenvectors.
\Statex \textbf{Output:} Projected signal matrix \( M  \in \mathbb{R}^{n \times m}\). Scaled eigenvalue vector \( \boldsymbol{\Lambda}_m = [\lambda_1, \lambda_2, \ldots, \lambda_m]^\top  \in \mathbb{R}^{m}\) that contains the top \( m \) scaled eigenvalues.

\State \(K \gets \texttt{computeKernelMatrix}(X)\) \Comment{Compute the kernel matrix \( K_{i,i'} = k(\boldsymbol{x}_i, \boldsymbol{x}_{i'}) \)}
\State \(K_{\text{centered}} \gets K - \mathbf{1}_n \cdot K - K \cdot \mathbf{1}_n + \mathbf{1}_n \cdot K \cdot \mathbf{1}_n\) \Comment{Center the kernel matrix}
\State \( \boldsymbol{\Lambda}_m, V_m \gets \texttt{solveKPCA}(K_{\text{centered}}, m) \) 
\Comment{Solve KPCA for the top \( m \) eigenvalues and eigenvectors, where \( V_m = [\boldsymbol{\alpha}_1, \ldots, \boldsymbol{\alpha}_m] \)}
\State \(M = K_{\text{centered}} \cdot V_m\) \Comment{Project data using the top \( m \) eigenvectors}
\State \Return \( M\), \( \boldsymbol{\Lambda}_m \) \Comment{Return projected signals and top \( m \) scaled eigenvalues}
\end{algorithmic}
\end{algorithm}

\subsection{Weighted Knockoff for Multiple Objectives}

This section introduces the \emph{weighted knockoff for multiple objectives} used in \TheName{} to perform feature selection for multiple signal objectives.

\subsubsection{Knockoff Matrix Generation}
First, the knockoff matrix \(\tilde{\Psi}\) is generated as a synthetic copy of the polynomial feature matrix \(\Psi \in \mathbb{R}^{n \times d_{max}}\), as described in line 1 of Algorithm~\ref{alg:WEKO}. 

\subsubsection{Knockoff Statistics Computation}
For each of the \(m\) self-supervised signal column vectors \( \boldsymbol{z}_j \) from the projected signal matrix \( M \), knockoff statistics \(\boldsymbol{W}^{(j)} \in \mathbb{R}^{d_{max}} \) are computed by comparing \(\Psi\) with \(\tilde{\Psi}\) (lines 2--4 in Algorithm~\ref{alg:WEKO}). Specifically, the \(j\)-th signal vector \( \boldsymbol{z}_j \) is treated as the dependent variable, while \(\Psi\) and \(\tilde{\Psi}\) are used as independent variables. Each component \( W^{(j)}_d \) represents the knockoff statistic for the \(d\)-th feature in \(\Psi\), quantifying its importance for predicting \( \boldsymbol{z}_j \), as described in \eqref{eq:knockoff_statistics}.

\subsubsection{Weighted Significance Score Calculation}\label{sec:weighted_significance}
Finally, the weighted significance score vector \(\boldsymbol{s} \in \mathbb{R}^{d_{max}}\) is obtained by combining the knockoff statistics across all signals using the scaled eigenvalue vector \(\boldsymbol{\Lambda}_m\) (lines 5-6 in Algorithm~\ref{alg:WEKO}). This results in:
\begin{equation}
\boldsymbol{s} = W \cdot \boldsymbol{\Lambda}_m = \sum_{j=1}^{m} \lambda_j \boldsymbol{W}^{(j)}
\end{equation}
where \( W = [\boldsymbol{W}^{(1)}, \boldsymbol{W}^{(2)}, \ldots, \boldsymbol{W}^{(m)}] \in \mathbb{R}^{d_{max} \times m} \) is the matrix of knockoff statistics for the \(m\) signals, and \(\boldsymbol{\Lambda}_m = [\lambda_1, \ldots, \lambda_m]^T\) is the vector of scaled eigenvalues from the \emph{statistical self-supervised signal generation} step. 

\subsubsection{Interpretation of Weighted Scores}
The idea behind this approach is that the scaled eigenvalues represent the variance in different signal directions, and weighting the knockoff statistics by these eigenvalues ensures that directions with greater variance are given more emphasis in the final weighted score, which is then utilized for further analysis. 

\begin{algorithm}
\caption{\texttt{WEKO}: Weighted Knockoff for Multiple Objectives}
\label{alg:WEKO}
\begin{algorithmic}[1]
\Statex \textbf{Input:} Polynomial feature matrix \(\Psi \in \mathbb{R}^{n \times d_{max}}\). Projected signal matrix \( M = [\boldsymbol{z}_1, \ldots, \boldsymbol{z}_m]  \in \mathbb{R}^{n \times m}\), where \(\boldsymbol{z}_j\) is the \(j\)-th column vector. 
Scaled eigenvalue vector \( \boldsymbol{\Lambda}_m = [\lambda_1, \ldots, \lambda_m]  \in \mathbb{R}^{m}\).
\Statex \textbf{Output:} Weighted knockoff significance score vector \(\boldsymbol{s} \in \mathbb{R}^{d_{max}}\) for features in \(\Psi\).

\State \(\tilde{\Psi} \gets \texttt{computeKnockoff}(\Psi)\) \Comment{Generate knockoff matrix \(\tilde{\Psi}\)}

\For{each signal column \( \boldsymbol{z}_j \) in \( M \)}
    \State \(\boldsymbol{W}^{(j)} \gets \texttt{KnockoffStats}(\Psi, \tilde{\Psi}, \boldsymbol{z}_j)\) \Comment{Compute knockoff statistics}
\EndFor

\State \( W = [\boldsymbol{W}^{(1)}, \ldots, \boldsymbol{W}^{(m)}]\) \Comment{Stack knockoff statistics}
\State \(\boldsymbol{s} \gets W \cdot \boldsymbol{\Lambda}_m\) \Comment{Compute weighted significance scores}
\State \Return \(\boldsymbol{s}\) \Comment{Return significance score vector \(\boldsymbol{s}\)}
\end{algorithmic}
\end{algorithm}

\subsection{$p$-values Estimating for Polynomial Features}\vspace{-1mm}
This section outlines the procedure for estimating \(p\)-values for polynomial features. 
After obtaining feature importance scores, practices for p-value calibration include parametric and non-parametric methods.

\begin{algorithm}
\caption{\texttt{estPVal}: $p$-values Estimation for Polynomial Features}
\label{alg:estPVal}
\begin{algorithmic}[1]
\Statex \textbf{Input:} Weighted knockoff significance score vector \(\boldsymbol{s} = [s_1, s_2, \ldots, s_{d_{max}}]^\top \in \mathbb{R}^{d_{max}}\) for polynomial features in matrix \(\Psi\).
\Statex \textbf{Parameters:} Distribution family \(\mathcal{P}(\mu, \sigma ^2)\) for the scores \(s_d\), where \(\mu\) is the mean and \(\sigma^2\) is the variance.
\Statex \textbf{Output:} $p$-value vector \(\hat{\boldsymbol{p}} = [\hat{p}_1, \hat{p}_2, \ldots, \hat{p}_{d_{max}}]^\top \in \mathbb{R}^{d_{max}}\) for weighted scores \(s_d\).

\State \(\hat{\mu} \gets \texttt{mean}(\boldsymbol{s})\), \(\hat{\sigma}^2 \gets \texttt{variance}(\boldsymbol{s})\) \Comment{Estimate mean \(\hat{\mu}\) and variance \(\hat{\sigma}^2\) from \(\boldsymbol{s}\)}

\State \(\hat{\mathcal{P}} \gets \texttt{fitDistribution}(\mathcal{P}, \hat{\mu}, \hat{\sigma}^2)\) \Comment{Fit the distribution \(\hat{\mathcal{P}}\) using \(\hat{\mu}\) and \(\hat{\sigma}^2\)}

\For{each feature index \(d = 1, 2, \ldots, d_{max}\)}
    \State \(\hat{p}_d \gets 1 - F_{\hat{\mathcal{P}}}(s_d)\) \Comment{Compute $p$-value \(\hat{p}_d\) using the CDF of \(\hat{\mathcal{P}}\)}
\EndFor

\State \Return \(\hat{\boldsymbol{p}} = [\hat{p}_1, \hat{p}_2, \ldots, \hat{p}_{d_{max}}]^\top\) \Comment{Return the vector of estimated $p$-values}
\end{algorithmic}
\end{algorithm}

The parametric approach assumes a specific distribution for the significance scores. Distributions such as normal, log-normal, or exponential may be considered, depending on prior knowledge regarding feature importance. For example, if most features are believed to be either unimportant or weakly important, with only a few being strongly significant, a log-normal distribution may be appropriate. This is because the log-normal distribution can effectively model the scenario where most scores are clustered near zero but a small number of scores are significantly larger, which often occurs when features have a small subset of highly influential variables. The log-normal distribution assumes that the log-transformed scores \(\log(s_d)\) follow a normal distribution, and the parameters \(\mu\) and \(\sigma^2\) can be estimated from the data as follows:
\begin{equation}
\hat{\mu} = \frac{1}{d_{max}} \sum_{d=1}^{d_{max}} \log(s_d), \quad \hat{\sigma}^2 = \frac{1}{d_{max}-1} \sum_{d=1}^{d_{max}} \left(\log(s_d) - \hat{\mu}\right)^2.
\end{equation}
Once the distribution \(\hat{\mathcal{P}}\) is fitted, the \(p\)-values for each feature are computed using the cumulative distribution function \(F\): \(\hat{p}_d = 1 - F_{\hat{\mathcal{P}}}(s_d)\), for \(d = 1, 2, \ldots, d_{max}\). With the \(p\)-values estimated via this parametric approach, they can then be used to rank features by statistical significance. 

\emph{p-Value Calibration.} However, in cases where the distribution of the scores is difficult to ascertain or does not align well with common distributions, a non-parametric approach provides an alternative. This approach does not require specific assumptions about the distribution of the scores, instead using methods like kernel density estimation (KDE)\cite{kde} to smooth the data and estimate the probability density function, or percentile-based approaches \cite{multipleKnockoff} to directly compute empirical percentiles and determine appropriate cut-offs for significance. Specifically, the p-value for each feature indexed by \(d\) can be computed as:
\begin{equation}\label{eq:pvalue}
\hat{p}_d = \frac{1 + \#\{d': s_{d'} \le s_d\}}{d_{\max}},
\end{equation}
where \( \#\{d': s_{d'} \le s_d\} \) counts how many features have scores less than or equal to \(s_d\). These methods are particularly useful when the distribution of the scores is unknown or irregular.

\subsection{Theoretical Analysis of \TheName{}}
In this section, we present the theoretical analysis of \TheName{}, highlighting the importance of the polynomial features identified by \TheName{} for the $j$-th principal component (PC). Additionally, we discuss how their statistical significance contributes to approximating the sparse polynomial feature representations of kernel PCA. Note that in Section~\ref{sec:weighted_significance}, we introduce a step to aggregate the significance of each polynomial feature across all kernel principal components by means of a weighted sum. In the present analysis, however, we focus exclusively on the polynomial features identified for each individual kernel principal component, rather than on their aggregated significance.

\subsubsection{Statistical Significance of Identified Polynomial Features}
Here, we present a theorem that guarantees FDR control when selecting polynomial terms for each kernel principal component in \TheName{}. Building on the hypothesis framework introduced in Section~\ref{subsection:hypothesisSection}, this theorem ensures that the statistical inference performed through the Knockoff procedure is both reliable and robust in identifying significant features.

\begin{theorem}[FDR Control for Significant Polynomial Features]
Suppose the signal index \(j\) is fixed. Let \( q \in [0, 1] \) be a target false discovery rate. Applying the Knockoff procedure with the threshold \(\tau\) defined in \eqref{eq:knockoff threshold} to select significant polynomial features \(\psi_d\), where \( z_{ij} \) are the response variables and \(\psi_d(\boldsymbol{x}_i)\) are the independent variables, guarantees FDR control with respect to the hypotheses \( H^{(j,d)}_0 \) and \( H^{(j,d)}_1 \) (see Definition \ref{def:hypothesis}) as follows:
\[
\text{FDR} = \mathbb{E}\left[\frac{\# \{ d : \beta_d^{*(j)} = 0 \text{ and } d \in \hat{A}_j \}}{\# \{ d : d \in \hat{A}_j \} \lor 1} \right] \leq q,
\]
where \( \hat{A}_j = \{ d : W^{(j)}_d \geq \tau \} \) is the set of selected indices and \(\beta_d^{*(j)}\) is defined in \eqref{eq:optimization}. 
\end{theorem}

\begin{proof}
Under the null hypotheses \(H^{(j,d)}_0\), \(W^{(j)}_d\) are independently and identically distributed (i.i.d.) and symmetric around 0, resulting in the following relationship:
\[
\frac{\# \{ d : \beta_d^{*(j)} = 0 \text{ and } W^{(j)}_d \leq t \}}{1 + \# \{ d : \beta_d^{*(j)} = 0 \text{ and } W^{(j)}_d \geq t \}} \sim \frac{\# \{ d : \beta_d^{*(j)} = 0 \text{ and } W^{(j)}_d \geq t \}}{1 + \# \{ d : \beta_d^{*(j)} = 0 \text{ and } W^{(j)}_d \leq t \}}
\]
Assume \( |W^{(j)}_1| \geq |W^{(j)}_2| \geq \dots \geq |W^{(j)}_{d_{max}}| \), and the threshold \(\tau\) is chosen as in \eqref{eq:knockoff threshold}. The threshold \(\tau\) acts as a stopping time for the supermartingale:
\[
\frac{\# \{ d : \beta_d^{*(j)} = 0 \text{ and } W^{(j)}_d \leq -t \}}{1 + \# \{ d : \beta_d^{*(j)} = 0 \text{ and } W^{(j)}_d \geq t \}}.
\]

By the Optional Stopping Theorem, the expected value at \(t = \tau\) is bounded by:
\[
\mathbb{E} \left[ \frac{\# \{ d : \beta_d^{*(j)} = 0 \text{ and } W^{(j)}_d \leq -\tau \}}{1 + \# \{ d : \beta_d^{*(j)} = 0 \text{ and } W^{(j)}_d \geq \tau \}} \right]
\]
\[
\leq \mathbb{E} \left[ \frac{\# \{ d : \beta_d^{*(j)} = 0 \text{ and } W^{(j)}_d \leq 0 \}}{1 + \# \{ d : \beta_d^{*(j)} = 0 \text{ and } W^{(j)}_d \geq 0 \}} \right]
\]
\[
\leq \mathbb{E} \left[ \frac{\# \{ d : \beta_d^{*(j)} = 0 \text{ and } W^{(j)}_d \geq 0 \}}{1 + p_0 - \# \{ d : \beta_d^{*(j)} = 0 \text{ and } W^{(j)}_d \geq 0 \}} \right] \leq 1
\]
where this final step follows from a property of the binomial distribution with parameters \(p_0\) and \(1/2\).

Therefore, the FDR is controlled by:
\[
\text{FDR} = \mathbb{E}\left[\frac{\# \{ d : \beta_d^{*(j)} = 0 \text{ and } d \in \hat{A}_j \}}{\# \{ d : d \in \hat{A}_j \} \lor 1} \right]
\]
\[ = \mathbb{E} \left[ \frac{\# \{ d : \beta_d^{*(j)} = 0 \text{ and } W^{(j)}_d \geq \tau \}}{\# \{ d : W^{(j)}_d \geq \tau \} \lor 1} \right]
\]
\[
\leq \mathbb{E} \left[ \frac{\# \{ d : \beta_d^{*(j)} = 0 \text{ and } W^{(j)}_d \geq \tau \}}{1 + \# \{ d : \beta_d^{*(j)} = 0 \text{ and } W^{(j)}_d \leq -\tau \}} \cdot \frac{1 + \# \{ d : W^{(j)}_d \leq -\tau \}}{\# \{ d : W^{(j)}_d \geq \tau \} \lor 1} \right]
\]
\[
\leq \mathbb{E} \left[ \frac{\# \{ d : \beta_d^{*(j)} = 0 \text{ and } W^{(j)}_d \geq \tau \}}{1 + \# \{ d : \beta_d^{*(j)} = 0 \text{ and } W^{(j)}_d \leq -\tau \}} \cdot q \right] \leq q.
\]
For further details, refer to \cite{fixedKnockoff, modelXKnockoff}.
\end{proof}

\subsection{Approximation to the Sparse Polynomial Representation}
The following theorem establishes that FDR-controlled feature selection provides theoretical bounds on the estimation error, ensuring that the approximation to the sparse representation remains accurate while maintaining interpretability.

\begin{theorem}[Error Bounds with FDR Control for Kernel PCA Approximation]
Let $\boldsymbol{z}_j$ be the $j$-th kernel principal component and 
$
\hat{A}_j = \{ d : W_d^{(j)} \ge \tau\}
$
be the set of polynomial features selected via the Knockoff procedure with FDR level $q \in [0,1]$. The coefficient $\hat{\boldsymbol{\beta}}^{(j)}=(\hat{\beta}^{(j)}_1, \dots, \hat{\beta}^{(j)}_{d_{max}})$ is obtained by performing linear regression of $\boldsymbol{z}_j$ on the variables in $\hat{A}_j$. Let 
$
A^*_j = \{ d : \beta_d^{*(j)} \neq 0\}
$
be the set of truly important features. Let $\sigma(p,n)$ refer to the variance of errors in kernel principal component scores (defined in Eq.~(\ref{eq:error-define})), so that we assume 
$
z_{i,j} = \mathbb{E}[z_{i,j}] + \varepsilon_{i,j}, \quad \varepsilon_{i,j} \sim \mathcal{N}(0,\sigma(p,n)).
$
Then, with probability at least $1-\delta$, the estimation error satisfies:
\begin{equation}
    \sqrt{\frac{1}{n} \sum_{i=1}^n \left(\mathbb{E}[z_{i,j}] - \sum_{d \in \hat{A}_j} \hat{\beta}_d^{(j)}\, \psi_d(\boldsymbol{x}_i) \right)^2} \le \mathcal{E}^* + C \sigma(p,n)\sqrt{\frac{q \cdot \log\left(d_{max}/{\delta}\right)}{n}},
\end{equation}
where
$
\mathcal{E}^* = \sqrt{\frac{1}{n} \sum_{i=1}^n \left(\mathbb{E}[z_{i,j}] - \sum_{d \in A^*_j} \beta_d^{*(j)}\, \psi_d(\boldsymbol{x}_i) \right)^2},
$
$C > 0$ is a constant, $n$ is the sample size, $d_{max}$ is the total number of polynomial features, and $q$ is the target false discovery rate.
\end{theorem}

\begin{proof}
We start by writing the error of the sparse polynomial approximation for the $j$-th kernel principal component as
\begin{equation}\label{eq:error_decomp}
\mathcal{E}^2 = \frac{1}{n}\sum_{i=1}^n \left(\mathbb{E}[z_{i,j}] - \sum_{d\in\hat{A}_j} \hat{\beta}_d^{(j)} \,\psi_d(\boldsymbol{x}_i)\right)^2.
\end{equation}
Let $\boldsymbol{\beta}^{*(j)}$ denote the optimal coefficient vector that minimizes the error using the true support $A^*_j$, i.e.,
$
\mathcal{E}^{*2} = \frac{1}{n}\sum_{i=1}^n \left(\mathbb{E}[z_{i,j}] - \sum_{d\in A^*_j} \beta_d^{*(j)} \,\psi_d(\boldsymbol{x}_i)\right)^2.
$
Now, decompose the error in Equation~\eqref{eq:error_decomp} into two parts:
\begin{equation}
\begin{aligned}
    \mathbb{E}[z_{i,j}] - \sum_{d\in\hat{A}_j} \hat{\beta}_d^{(j)}\, \psi_d(\boldsymbol{x}_i) = & \Bigl(\mathbb{E}[z_{i,j}] - \sum_{d\in A^*_j} \beta_d^{*(j)}\, \psi_d(\boldsymbol{x}_i)\Bigr)\\
    & + \left(\sum_{d\in A^*_j} \beta_d^{*(j)}\, \psi_d(\boldsymbol{x}_i) - \sum_{d\in\hat{A}_j} \hat{\beta}_d^{(j)}\, \psi_d(\boldsymbol{x}_i)\right).
\end{aligned}
\end{equation}
Taking the squared error and applying the triangle inequality (as well as standard inequality $(\sqrt{a+b}\leq \sqrt{a}+\sqrt{b}$) yields
\begin{equation}
\mathcal{E} \le \mathcal{E}^* + \sqrt{\frac{1}{n} \sum_{i=1}^n \left(\sum_{d\in A^*_j} \beta_d^{*(j)}\, \psi_d(\boldsymbol{x}_i) - \sum_{d\in\hat{A}_j} \hat{\beta}_d^{(j)}\, \psi_d(\boldsymbol{x}_i)\right)^2}.
\end{equation}
Let us denote the additional error due to model selection by
\begin{equation}
    \Delta = \frac{1}{n} \sum_{i=1}^n \left(\sum_{d\in A^*_j} \beta_d^{*(j)}\, \psi_d(\boldsymbol{x}_i) - \sum_{d\in\hat{A}_j}\hat{\beta}_d^{(j)}\, \psi_d(\boldsymbol{x}_i)\right)^2.
\end{equation}
Thus,
$
\mathcal{E} \le \mathcal{E}^* + \sqrt{\Delta}.
$
It now remains to bound $\Delta$.

\noindent\textbf{Bound on $\Delta$:} 
For clarity, denote the true support by 
$
S = A^*_j,
$
and the selected support by
$
T = \hat{A}_j.
$
The error $\Delta$ can be decomposed as the sum of two contributions:
\begin{enumerate}
    \item $\Delta_1$: the error from omitting coordinates in $S\setminus T$ (false negatives),
    \item $\Delta_2$: the error from including extraneous coordinates in $T\setminus S$ (false positives).
\end{enumerate}
In particular, after rearranging the terms we have:
\begin{equation}
\begin{aligned}
    \sum_{d\in S} \beta_d^{*(j)}\, \psi_d(\boldsymbol{x}_i) - \sum_{d\in T} \hat{\beta}_d^{(j)}\, \psi_d(\boldsymbol{x}_i)
= &\underbrace{\sum_{d\in S\setminus T} \beta_d^{*(j)}\, \psi_d(\boldsymbol{x}_i)}_{\text{omission error}} + \underbrace{\sum_{d\in T\setminus S} \left(0 - \hat{\beta}_d^{(j)}\right) \psi_d(\boldsymbol{x}_i)}_{\text{false inclusion error}}\\
&+ \underbrace{\sum_{d\in S\cap T} \left(\beta_d^{*(j)} - \hat{\beta}_d^{(j)}\right) \psi_d(\boldsymbol{x}_i)}_{\text{estimation error}}.    
\end{aligned}
\end{equation}
For brevity, denote these three contributions in $\ell_2$ aggregated error by $\Delta_{\text{omission}}$, $\Delta_{\text{false}}$, and $\Delta_{\text{est}}$. Then 
\begin{equation}\small
\Delta \leq \frac{1}{n} \sum_{i=1}^n \Bigl\{ \underbrace{\Bigl(\sum_{d\in S\setminus T} \beta_d^{*(j)} \psi_d(\boldsymbol{x}_i)\Bigr)^2}_{\Delta_{\text{omission}}}
+ \underbrace{\Bigl(\sum_{d\in T\setminus S} \hat{\beta}_d^{(j)} \psi_d(\boldsymbol{x}_i)\Bigr)^2}_{\Delta_{\text{false}}}
+ \underbrace{\Bigl(\sum_{d\in S\cap T} (\beta_d^{*(j)} - \hat{\beta}_d^{(j)}) \psi_d(\boldsymbol{x}_i)\Bigr)^2}_{\Delta_{\text{est}}} \Bigr\}.
\end{equation}
The rigorous treatment proceeds by separately controlling each term:
  
\textbf{(1) Control via FDR:} The Knockoff procedure guarantees that the expected proportion of falsely included features,
$
\frac{|T \setminus S|}{|T|},
$
is bounded by $q$. By an application of Markov's inequality (and using standard multiple testing results; see, e.g., \cite{BarberCandes2015}), one has that with probability at least $1-\delta_1$ (where $\delta_1$ can be made arbitrarily small via appropriate calibration), the cardinality of $T\setminus S$ satisfies
$
|T\setminus S| \leq C_1\,q\,|T|
$
for some constant $C_1>0$. In turn, under suitable assumptions on the boundedness (or the restricted eigenvalue condition) of the design matrix $\mathbf{\Psi}$, it follows that 
\begin{equation}
    \Delta_{\text{false}} \le \frac{|T\setminus S|}{n}\, \max_{d}\|\psi_d\|_2^2 \cdot \max_{d\in T\setminus S} |\hat{\beta}_d^{(j)}|^2 \le C_2\, q\, \frac{\log\left(\frac{d_{max}}{\delta_1}\right)}{n}\, \sigma(p,n)^2,
\end{equation}
where we have used standard concentration inequalities on the least-squares estimates to bound $|\hat{\beta}_d^{(j)}|$ by $O\!\left(\sigma(p,n)\sqrt{\frac{\log(d_{max}/\delta_1)}{n}}\right)$. 

\textbf{(2) Standard Gaussian Concentration:} Recall that the observed PCA score is given by
$
z_{i,j} = \mathbb{E}[z_{i,j}] + \varepsilon_{i,j}, \quad \varepsilon_{i,j} \sim \mathcal{N}\bigl(0,\sigma(p,n)\bigr).
$
When estimating the polynomial coefficients via least squares, the error incurred in the coefficients (i.e., the estimation error on $S\cap T$) is amplified by the noise vector $\mathbf{\varepsilon}$. Under a standard restricted eigenvalue condition on $\mathbf{\Psi}$, classical concentration results (see, e.g., \cite{boucheron2003concentration}) yield that, with probability at least $1-\delta_2$, 
\begin{equation}
   \Delta_{\text{est}} \le C_3\, \sigma(p,n)^2\, \frac{\log\left(\frac{d_{max}}{\delta_2}\right)}{n}.
\end{equation}
A similar bound holds also for $\Delta_{\text{omission}}$ by noting that if true features are omitted then the associated error in the approximation is, at most, comparable to the estimation error on the missing coordinates, leading to
\begin{equation}
    \Delta_{\text{omission}} \le C_4\, \sigma(p,n)^2\, \frac{\log\left(\frac{d_{max}}{\delta_2}\right)}{n}.
\end{equation}

Constants $C_1, C_2, C_3, C_4$ depend on properties of the design matrix and the noise distribution, and may be absorbed into a single constant $C$. Combining the three bounds (and choosing $\delta_1,\delta_2$ so that $\delta_1+\delta_2 = \delta$), we deduce that there exists a constant $C > 0$ such that
\begin{equation}
    \sqrt{\Delta} \le C \cdot \sigma(p,n) \sqrt{\frac{q \cdot \log\left(\frac{d_{max}}{\delta}\right)}{n}}.
\end{equation}
Finally, plugging this bound into the earlier inequality, we obtain the claimed bound:
\begin{equation}
  \sqrt{\frac{1}{n}\sum_{i=1}^n \left(\mathbb{E}[z_{i,j}] - \sum_{d \in \hat{A}_j} \hat{\beta}_d^{(j)}\, \psi_d(\boldsymbol{x}_i)\right)^2} \le \mathcal{E}^* + C \cdot \sigma(p,n) \sqrt{\frac{q \cdot \log\left(\frac{d_{max}}{\delta}\right)}{n}},
\end{equation}
with probability at least $1-\delta$. We complete the proof.
\end{proof}

\subsection{Discussions and Limitations}
The theoretical analysis above establishes that by applying the Knockoff procedure to select polynomial features corresponding to each kernel principal component, \TheName{} achieves reliable control of the false discovery rate (FDR) while accurately capturing sparse, nonlinear feature interactions. This result implies that the extracted polynomial terms not only approximate the complex nonlinear relationships among the original features but also enable a principled reduction of high-dimensional kernels into interpretable interaction components, thus bridging explicit polynomial expansions and implicit kernel mappings (see~\cite{fixedKnockoff, modelXKnockoff}).

From an interpretative perspective, the derived error bounds guarantee that the sparse representations produced by \TheName{} faithfully approximate the true kernel PCA scores with quantifiable accuracy. This statistical assurance makes it possible to trust that the selected polynomial features are significant and interpretable, allowing practitioners to extract meaningful explanations from the model. The controlled estimation error and FDR underpin the reliability of both global and local interpretations, ensuring that the interactions uncovered are not merely artifacts of noise but are reflective of genuine underlying data structure.

Again, the above analysis guarantees FDR control in the selection of significant polynomial terms and approximation error for each single kernel principal component in \TheName{}. While these guarantees apply to the selection process at the individual component level, the overall performance of feature selection in \TheName{}, which relies on the weighted sum of knockoff statistics (in Section~\ref{sec:weighted_significance}), are evaluated in the experimental section. Afterall, \TheName{} uniquely synthesizes robust nonlinear feature interactions captured via kernel PCA with transparent interpretability enforced through rigorous, FDR-controlled Knockoff selection, striking a crucial balance for deploying trustworthy high-dimensional, high-stakes machine learning models. 

\section{Evaluation with Realistic Datasets}\label{section:experiments}
The performance of \TheName{} was evaluated based on three key aspects. First, the \emph{effectiveness} of feature selection was assessed in \ref{subsec:overallComparison} by comparing the prediction accuracy on testing datasets using features selected by \TheName{}, KPCA, SKPCA, and other supervised and unsupervised baselines. Second, the \emph{reconstruction} accuracy of selected features was evaluated in \ref{subsubsec:higgs_visual} by visually assessing how well the selected features captured the underlying structure of the kernel principal components (KPCs). Finally, the \emph{interpretability} aspect was analyzed in \ref{subsubsec:higgs_feature} and \ref{subsec:Superconductivity} to understand the interpretability of the features selected by \TheName{}. The experimental setup is described in \ref{subsection:setup}.

\begin{table}[ht]
  \centering
  \small
  \caption{Overview of Datasets Used for Model Evaluation}\label{tab:datasetsDescription}
  \begin{tabular}{@{}lccc@{}}
    \toprule
    \textbf{Dataset} & \textbf{\#Variables} & \textbf{\#Features} & \textbf{Task} \\
    \midrule
    Higgs~\cite{higgs_dataset} & 24 & 325 & Classification (2 classes, high-energy physics) \\
    Yahoo~\cite{yahoo} & 25 & 351 & Classification (5 classes, web search) \\
    Bias~\cite{bias_dataset} & 20 & 231 & Regression (weather forecasting) \\
    Facial~\cite{facial} & 25 & 351 & Regression (medical imaging)\\
    Song~\cite{Year} & 25 & 351 & Regression (music analysis)\\
    Superconductivity~\cite{superconductivity} & 25 & 351 & Regression (materials science)\\
    \bottomrule
  \end{tabular}
\end{table}



\subsection{Experiment Setup}\label{subsection:setup}
The experiments were conducted on six real-world datasets, as summarized in Table~\ref{tab:datasetsDescription}, which details the number of variables, the total number of polynomial features selected, and the nature of the tasks. Due to the high dimensionality, we first applied the Knockoff filter for variable selection before constructing polynomial features. The "Number of Variables" column reflects the remaining features after filtering. Each experiment was repeated at least ten times, with up to 15,000 samples randomly selected for training and at least 2,000 for testing across all datasets. \TheName{} used \texttt{knockpy} for Knockoff implementation, while KPCA, SKPCA, and other baseline methods were executed using \texttt{scikit-learn} with default parameters.

For the evaluation of prediction accuracy, fixed-length feature selection was employed in \ref{subsec:overallComparison} to ensure a fair comparison across methods by controlling the number of selected features. In contrast, \(p\)-values derived from percentiles were used for feature selection in \ref{subsec:Higgs}-\ref{subsec:Superconductivity} to assess interpretability. as it allows for a more flexible and significance-driven selection of features, which is better suited for understanding the relative importance of features.

\subsection{Overall Comparison}\label{subsec:overallComparison}

In our experiments, we assessed the performance of \TheName{} across several real-world datasets, with results presented in Table~\ref{tbl:test_acc Yahoo} to \ref{tbl:test_mse_superconductivity_baseline}. Each dataset was tested with four different kernel functions: cosine, radial basis function (RBF), sigmoid, and random Fourier features (RFF). For classification tasks, we used accuracy as the evaluation metric, while Mean Squared Error (MSE) was employed for regression tasks. Here, we present the results for the Yahoo, Higgs, Bias, and Superconductivity datasets, with additional results provided in the Supplementary materials (Online Resource 1).  

For feature selection, a fixed-length approach was applied to the polynomial features, with prediction results reported for models using 10, 20, 50, 100, and 150 selected terms. The predictive performance was evaluated using five models: Support Vector Classifier (SVC), Random Forest (RF), Logistic Regression (Log), Multi-layer Perceptron Classifier (MLPC), and XGBoost (XGB), all implemented using \texttt{scikit-learn} with default parameters.

To assess the impact on prediction accuracy of the hyperparameter \(m\), which represents the number of kernel principal components (KPCs) used in \TheName{}, Figure~\ref{fig:barPlots} presents the accuracy for different numbers of polynomial terms selected by \TheName{} (x-axis, \(r = 10, 20, 50, 100, 150\)) with varying \(m = 1, 2, 3, 50\) (denoted by different bar labels). Each set of bars shows the performance for different \(m\) values under the same \(r\), highlighting how the number of KPCs affects accuracy. The figure displays results for the Yahoo dataset using the sigmoid kernel, with similar patterns observed across other datasets and kernels. Since \(m\) had minimal impact, results are consistently reported for \(m = 50\).

\begin{figure}[htbp]
    \centering
    \subfloat[SVC]{
        \includegraphics[width=0.33\textwidth]{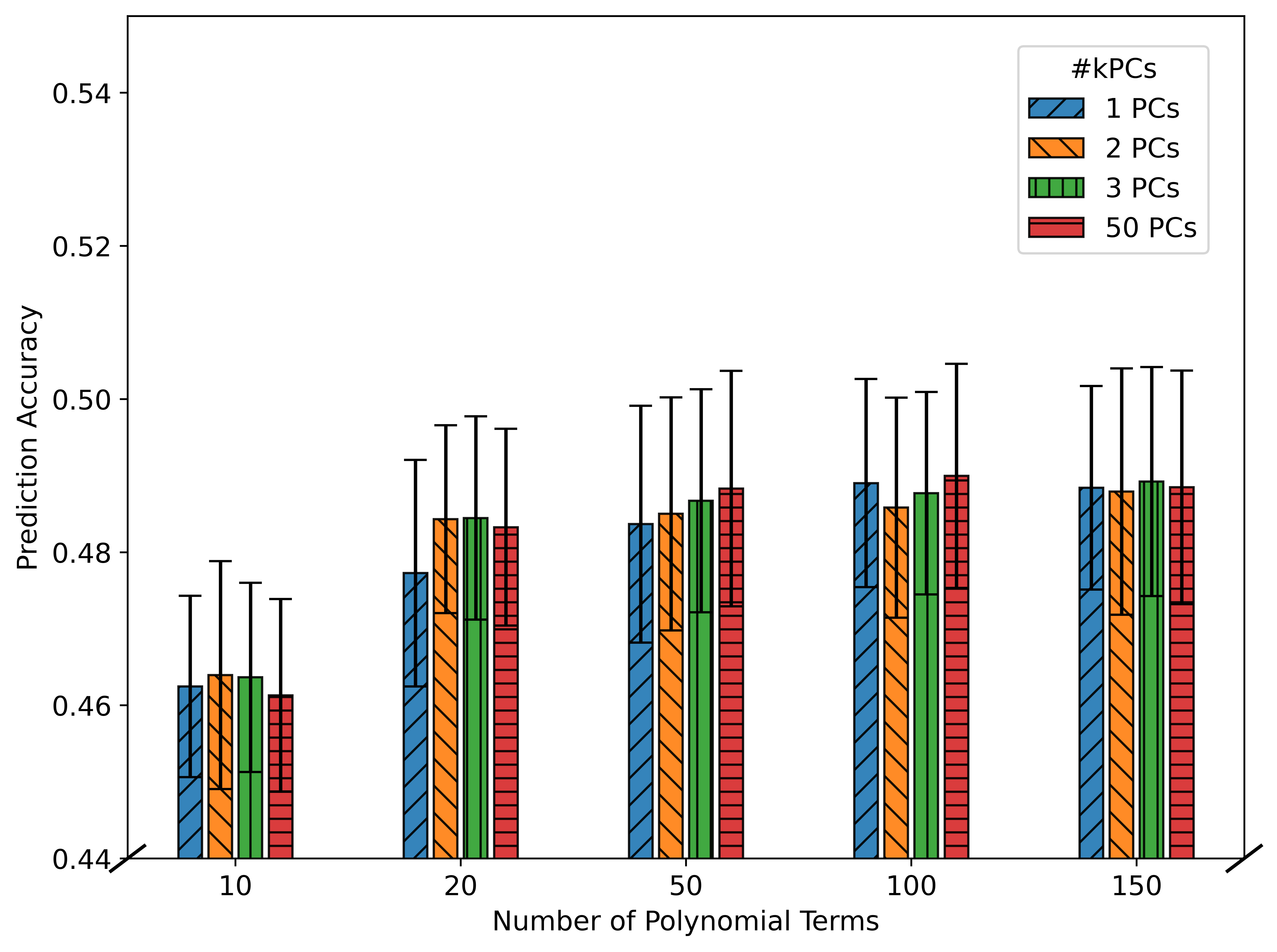}
        \label{fig1a}
    }
    \subfloat[RF]{
        \includegraphics[width=0.33\textwidth]{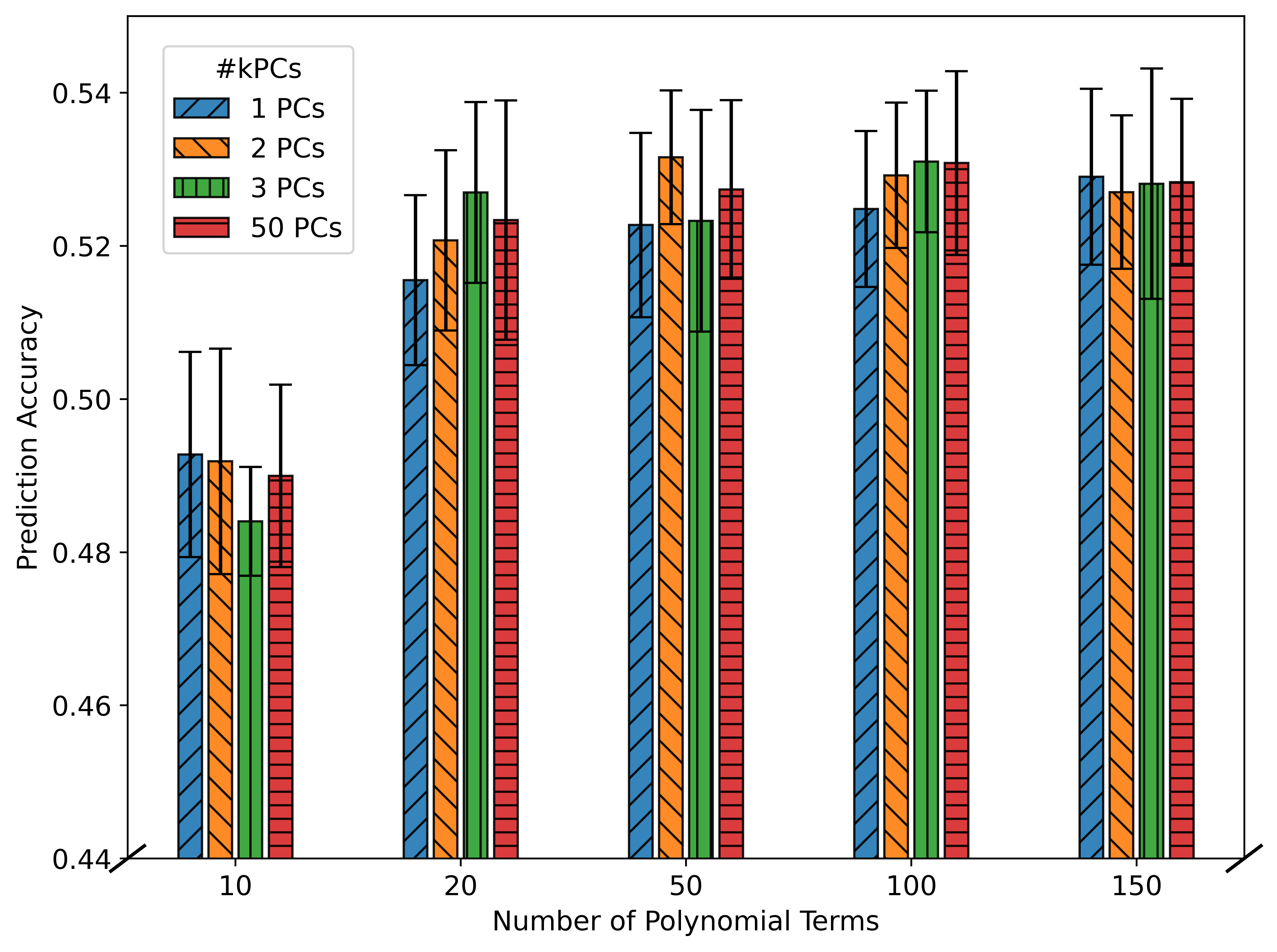}
        \label{fig1b}
    }
    \subfloat[Log]{
        \includegraphics[width=0.33\textwidth]{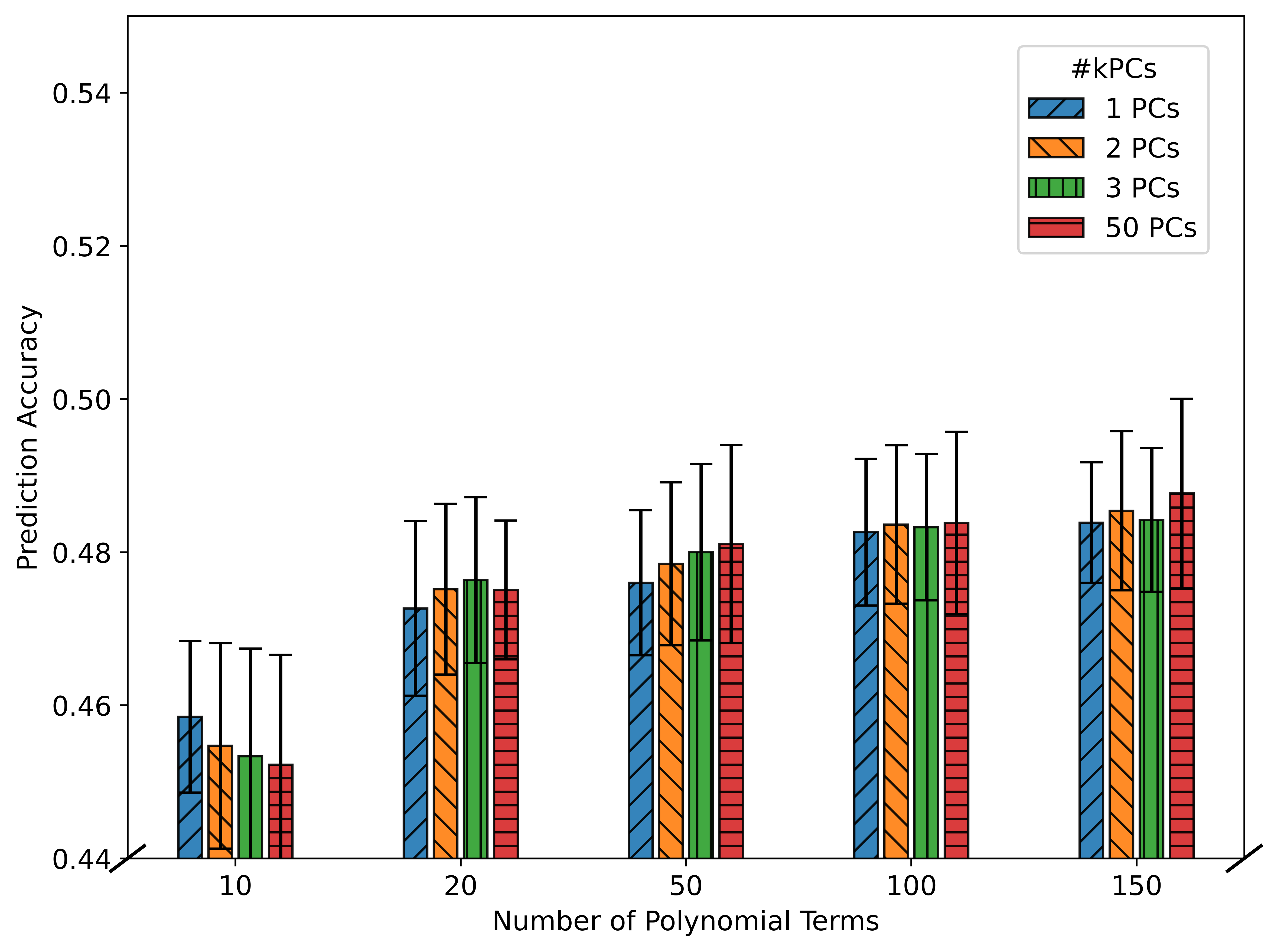}
        \label{fig1c}
    }\\
    \subfloat[MLPC]{
        \includegraphics[width=0.33\textwidth]{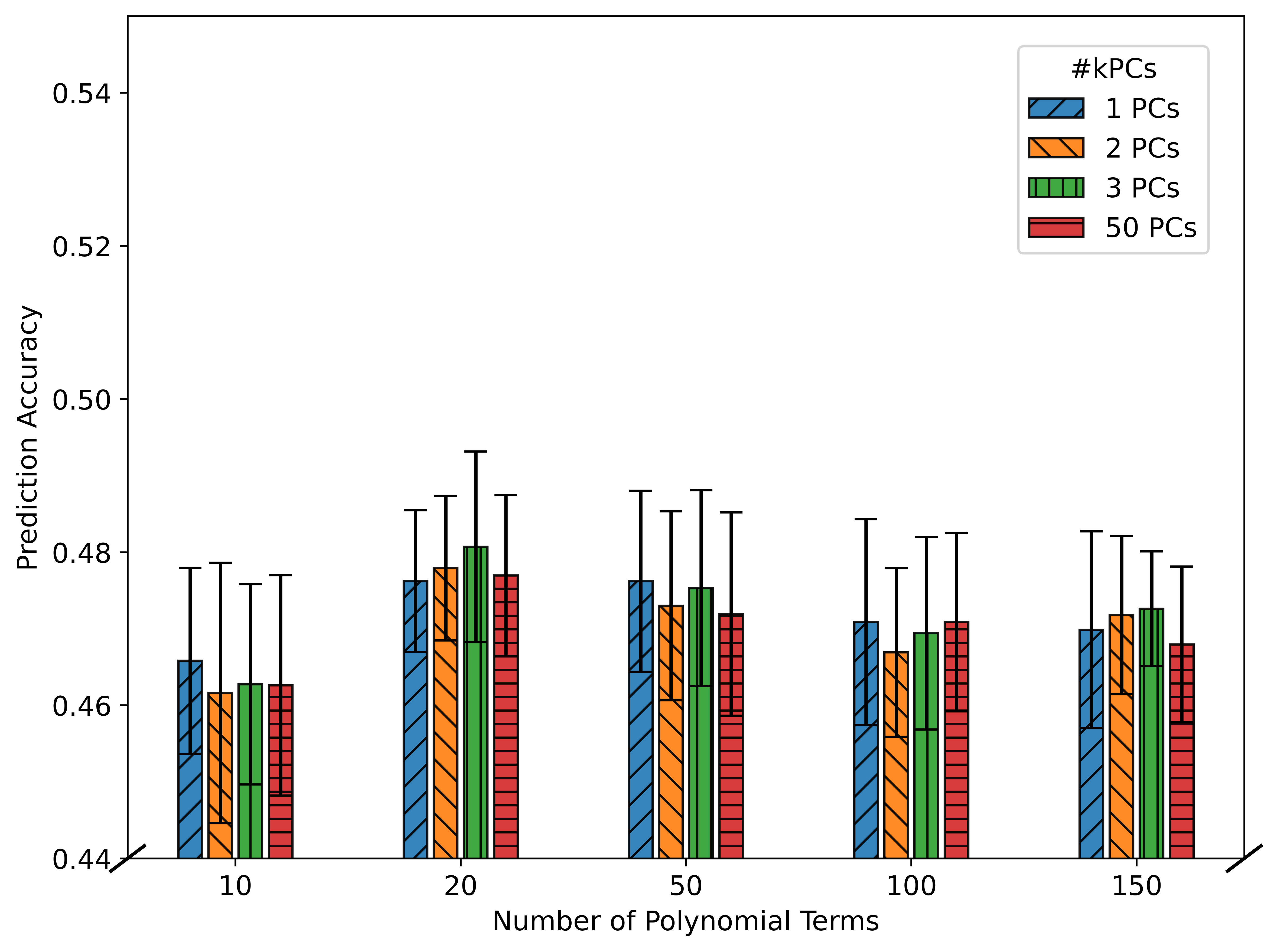}
        \label{fig1d}
    }
    \subfloat[XGBoost]{
        \includegraphics[width=0.33\textwidth]{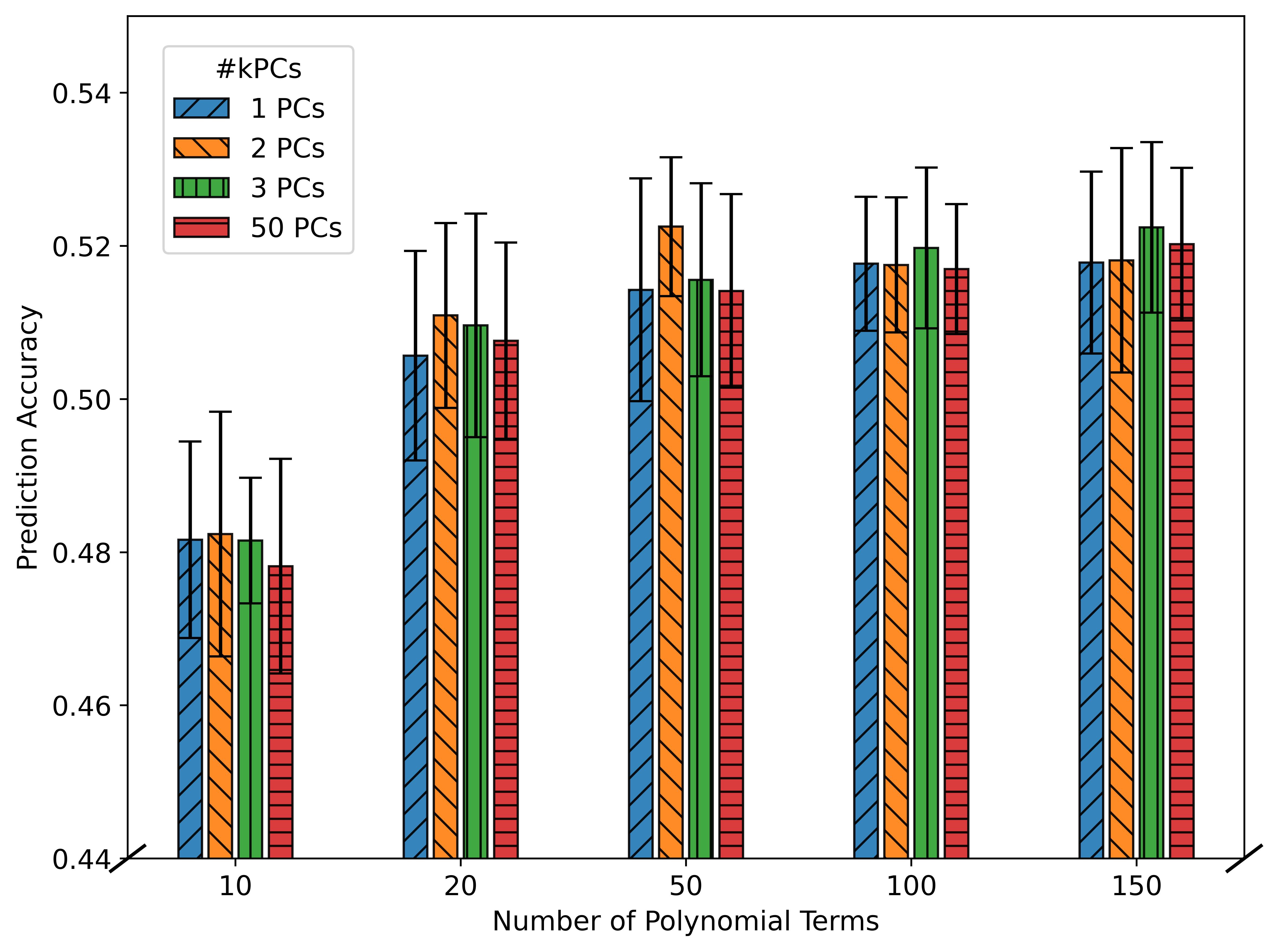}
        \label{fig1e}
    }
    \caption{Prediction accuracy using features selected by \TheName{} based on the first \(m\) kernel principal components (\(m = 1, 2, 3, 50\)) with different prediction models}
    \label{fig:barPlots}
\end{figure}

\begin{table}
    \centering
    \caption{Comparison of Prediction Accuracy Between \TheName{}, KPCA, and SKPCA on the Yahoo Testing Dataset (values at the \textbf{first places} and the \underline{second places})}\label{tbl:test_acc Yahoo} 
\footnotesize	
\begin{tabular}{c|cc|ccccc}
	\toprule
 	method & kernel & \# terms & SVC & RF & Log. Reg. & MLPC & XGB \\
	\midrule
	\multirow{5}{*}{\TheName} & \multirow{5}{*}{Cosine} & \multicolumn{1}{|c|}{10} & $0.453 \pm 0.016$ & $0.484 \pm 0.015$ & $0.448 \pm 0.013$ & $0.459 \pm 0.016$ & $0.477 \pm 0.010$\\
	& & \multicolumn{1}{|c|}{20} & $0.484 \pm 0.014$ & $0.527 \pm 0.011$ & $0.473 \pm 0.011$ & $0.479 \pm 0.011$ & $0.511 \pm 0.012$\\
	& & \multicolumn{1}{|c|}{50} & $0.486 \pm 0.015$ & $0.528 \pm 0.009$ & $0.481 \pm 0.010$ & $0.478 \pm 0.010$ & $\mathbf{0.521 \pm 0.010}$ \\
	& & \multicolumn{1}{|c|}{100} & $0.489 \pm 0.014$ & $0.526 \pm 0.010$ & $0.484 \pm 0.010$ & $0.465 \pm 0.012$ & $\mathbf{0.521 \pm 0.009}$ \\
	& & \multicolumn{1}{|c|}{150} & $0.490 \pm 0.014$ & $0.528 \pm 0.011$ & $ \mathbf{0.488 \pm 0.010}$ & $0.472 \pm 0.014$ & $0.518 \pm 0.010$ \\
	\midrule
	\multirow{5}{*}{\TheName} & \multirow{5}{*}{RBF} & \multicolumn{1}{|c|}{10} & $0.462 \pm 0.014$ & $0.490 \pm 0.017$ & $0.454 \pm 0.009$ & $0.466 \pm 0.014$ & $0.487 \pm 0.012$\\
	& & \multicolumn{1}{|c|}{20} & $0.476 \pm 0.014$ & $0.518 \pm 0.009$ & $0.473 \pm 0.013$ & $0.475 \pm 0.009$ & $0.507 \pm 0.014$\\
	& & \multicolumn{1}{|c|}{50} & $0.485 \pm 0.015$ & $0.521 \pm 0.010$ & $0.482 \pm 0.016$ & $0.466 \pm 0.010$ & $0.512 \pm 0.011$ \\
	& & \multicolumn{1}{|c|}{100} & $0.490 \pm 0.014$ & $0.524 \pm 0.007$ & $0.483 \pm 0.017$ & $0.469 \pm 0.010$ & $0.514 \pm 0.011$ \\
	& & \multicolumn{1}{|c|}{150} & $\mathbf{0.492 \pm 0.014}$ & $0.527 \pm 0.013$ & $\mathbf{0.488 \pm 0.013}$ & $0.466 \pm 0.014$ & $0.512 \pm 0.014$ \\
	\midrule
	\multirow{5}{*}{\TheName} & \multirow{5}{*}{Sigmoid} & \multicolumn{1}{|c|}{10} & $0.461 \pm 0.013$ & $0.490 \pm 0.012$ & $0.452 \pm 0.014$ & $0.463 \pm 0.014$ & $0.478 \pm 0.014$\\
	& & \multicolumn{1}{|c|}{20} & $0.483 \pm 0.013$ & $0.523 \pm 0.016$ & $0.475 \pm 0.009$ & $0.477 \pm 0.010$ & $0.508 \pm 0.013$\\
	& & \multicolumn{1}{|c|}{50} & $0.488 \pm 0.015$ & $0.527 \pm 0.012$ & $0.481 \pm 0.013$ & $0.472 \pm 0.013$ & $0.514 \pm 0.013$ \\
	& & \multicolumn{1}{|c|}{100} & $0.490 \pm 0.015$ & $\mathbf{0.531 \pm 0.012}$ & $0.484 \pm 0.012$ & $0.471 \pm 0.012$ & $0.517 \pm 0.008$ \\
	& &  \multicolumn{1}{|c|}{150} & $0.488 \pm 0.015$ & $0.528 \pm 0.011$ & $\mathbf{0.488 \pm 0.012}$ & $0.468 \pm 0.010$ & $\underline{0.520 \pm 0.010}$ \\
	\midrule
	\multirow{5}{*}{\TheName} & \multirow{5}{*}{RFF} & \multicolumn{1}{|c|}{10} & $0.463 \pm 0.018$ & $0.490 \pm 0.011$ & $0.454 \pm 0.013$ & $0.461 \pm 0.016$ & $0.487 \pm 0.014$ \\
	& & \multicolumn{1}{|c|}{20} & $0.479 \pm 0.014$ & $0.519 \pm 0.011$ & $0.472 \pm 0.012$ & $0.472 \pm 0.011$ & $0.507 \pm 0.014$ \\
	& & \multicolumn{1}{|c|}{50} & $0.485 \pm 0.014$ & $0.524 \pm 0.011$ & $0.479 \pm 0.010$ & $0.468 \pm 0.013$ & $0.516 \pm 0.010$ \\
	& & \multicolumn{1}{|c|}{100} & $0.488 \pm 0.013$ & $\underline{0.530 \pm 0.011}$ & $0.483 \pm 0.011$ & $0.471 \pm 0.014$ & $0.517 \pm 0.012$ \\
	& &  \multicolumn{1}{|c|}{150} & $0.490 \pm 0.014$ & $\mathbf{0.531 \pm 0.013}$ & $\underline{0.487 \pm 0.012}$ & $0.471 \pm 0.009$ & $0.516 \pm 0.010$ \\
	\midrule
	\multicolumn{1}{c|}{} & \multicolumn{2}{c|}{Cosine} & $0.482 \pm 0.014$ & $0.496 \pm 0.012$ & $0.474 \pm 0.013$ & $0.475 \pm 0.007$ & $0.479 \pm 0.012$\\
	\multicolumn{1}{c|}{KPCA} & \multicolumn{2}{c|}{RBF} & $\underline{0.491 \pm 0.016}$ & $0.517 \pm 0.013$ & $0.477 \pm 0.015$ & $\underline{0.480 \pm 0.013}$ & $0.509 \pm 0.013$ \\
	 \multicolumn{1}{c|}{} & \multicolumn{2}{c|}{Sigmoid} & $\underline{0.491 \pm 0.017}$ & $0.521 \pm 0.013$ & $0.471 \pm 0.012$ & $\mathbf{0.485 \pm 0.015}$ & $0.504 \pm 0.014$ \\
	\midrule
	\multicolumn{1}{c|}{} & \multicolumn{2}{c|}{Cosine} & $0.482 \pm 0.014$ & $0.512 \pm 0.014$ & $0.474 \pm 0.013$ & $0.473 \pm 0.009$ & $0.498 \pm 0.014$ \\
	\multicolumn{1}{c|}{SKPCA} & \multicolumn{2}{c|}{RBF} & $\underline{0.491 \pm 0.016}$ & $0.519 \pm 0.014$ & $0.476 \pm 0.015$ & $\mathbf{0.485 \pm 0.012}$ & $0.508 \pm 0.011$ \\
	 \multicolumn{1}{c|}{} & \multicolumn{2}{c|}{Sigmoid} & $0.489 \pm 0.014$ & $0.516 \pm 0.012$ & $0.470 \pm 0.013$ & $\underline{0.480 \pm 0.012}$ & $0.501 \pm 0.014$ \\
	\bottomrule
\end{tabular}
\end{table}

\begin{table}
    \centering
    \footnotesize
    \caption{Prediction Accuracy of Additional Baseline Methods on the Yahoo Testing Dataset (values at the \textbf{first places} and the \underline{second places})}\label{tbl:test_acc Yahoo baseline} 
	\begin{tabular}{c|c|ccccc}
	\toprule
 	\multicolumn{2}{c}{Method\slash Model} & SVC & RF & Log & MLPC & XGB \\
	\midrule
	\multicolumn{2}{c|}{Original Variables} & $0.489 \pm 0.017$ & $\mathbf{0.528 \pm 0.009}$ & $0.478 \pm 0.012$ & $\mathbf{0.480 \pm 0.013}$ & $\mathbf{0.519 \pm 0.014}$ \\
	\midrule
	\multirow{3}{*}{Unsupervised} & PCA & $\underline{0.492 \pm 0.017}$ & $\underline{0.525 \pm 0.011}$ & $0.477 \pm 0.012$ & $\mathbf{0.480 \pm 0.014}$ & $\underline{0.516 \pm 0.013}$ \\
	& ICA & $\mathbf{0.497 \pm 0.013}$ & $0.516 \pm 0.014$ & $\underline{0.480 \pm 0.012}$ & $0.460 \pm 0.010$ & $0.504 \pm 0.012$\\
	& SVD & $0.462 \pm 0.015$ & $0.460 \pm 0.008$ & $0.458 \pm 0.013$ & $0.460 \pm 0.012$ & $0.449 \pm 0.008$ \\
	\midrule
	\multirow{2}{*}{Supervised} & RFE-poly & $0.490 \pm 0.014$ & $\mathbf{0.528 \pm 0.011}$ & $\mathbf{0.487 \pm 0.012}$ & $0.468 \pm 0.008$ & $0.515 \pm 0.009$ \\
	& RFE-var & $0.479 \pm 0.011$ & $0.517 \pm 0.012$ & $0.472 \pm 0.009$ & $\underline{0.477 \pm 0.013}$ & $0.507 \pm 0.015$\\
	\bottomrule
	\end{tabular}
\end{table}

\begin{table}
    \centering
    \footnotesize
    \caption{Comparison of Prediction Accuracy Between \TheName{}, KPCA, and SKPCA on the Higgs Testing Dataset (values at the \textbf{first places} and the \underline{second places})}\label{tbl:test_acc Higgs} 
	\begin{tabular}{c|cc|ccccc}
	\toprule
 	method & kernel & (\# terms) & SVC & RF & log & MLPC & XGB \\
	\midrule
	\multirow{5}{*}{\TheName} & \multirow{5}{*}{Cosine} & \multicolumn{1}{|c|}{10} & $0.563 \pm 0.002$ & $0.645 \pm 0.004$ & $0.527 \pm 0.004$ & $0.559 \pm 0.005$ & $0.622 \pm 0.005$\\
	& & \multicolumn{1}{|c|}{20} & $0.566 \pm 0.003$ & $0.648 \pm 0.005$ & $0.526 \pm 0.004$ & $0.566 \pm 0.004$ & $0.623 \pm 0.004$\\
	& & \multicolumn{1}{|c|}{50} & $0.618 \pm 0.009$ & $0.712 \pm 0.017$ & $0.576 \pm 0.018$ & $0.634 \pm 0.014$ & $0.697 \pm 0.017$\\
	& & \multicolumn{1}{|c|}{100} & $0.637 \pm 0.004$ & $0.741 \pm 0.006$ & $0.597 \pm 0.006$ & $0.668 \pm 0.004$ & $0.729 \pm 0.007$\\
	& & \multicolumn{1}{|c|}{150} & $0.642 \pm 0.004$ & $\underline{0.751 \pm 0.006}$ & $0.615 \pm 0.002$ & $\underline{0.686 \pm 0.003}$ & $\underline{0.746 \pm 0.006}$\\
	\midrule
	\multirow{5}{*}{\TheName} & \multirow{5}{*}{RBF} & \multicolumn{1}{|c|}{10} & $0.563 \pm 0.005$ & $0.646 \pm 0.004$ & $0.527 \pm 0.004$ & $0.560 \pm 0.005$ & $0.621 \pm 0.004$\\
	& & \multicolumn{1}{|c|}{20} & $0.564 \pm 0.006$ & $0.647 \pm 0.005$ & $0.526 \pm 0.004$ & $0.563 \pm 0.007$ & $0.625 \pm 0.004$\\
	& & \multicolumn{1}{|c|}{50} & $0.637 \pm 0.009$ & $0.735 \pm 0.015$ & $0.611 \pm 0.010$ & $0.667 \pm 0.010$ & $0.731 \pm 0.012$ \\
	& & \multicolumn{1}{|c|}{100} & $0.637 \pm 0.009$ & $0.735 \pm 0.015$ & $0.611 \pm 0.010$ & $0.667 \pm 0.010$ & $0.731 \pm 0.012$\\
	& & \multicolumn{1}{|c|}{150} & $\underline{0.646 \pm 0.008}$ & $\underline{0.751 \pm 0.009}$ & $\underline{0.628 \pm 0.009}$ & $\underline{0.686 \pm 0.004}$ & $0.745 \pm 0.009$\\
	\midrule
	\multirow{5}{*}{\TheName} & \multirow{5}{*}{Sigmoid} & \multicolumn{1}{|c|}{10} & $0.562 \pm 0.004$ & $0.646 \pm 0.004$ & $0.527 \pm 0.004$ & $0.558 \pm 0.005$ & $0.619 \pm 0.005$\\
	& & \multicolumn{1}{|c|}{20} & $0.567 \pm 0.004$ & $0.646 \pm 0.004$ & $0.526 \pm 0.004$ & $0.562 \pm 0.007$ & $0.624 \pm 0.005$\\
	& & \multicolumn{1}{|c|}{50} & $0.613 \pm 0.010$ & $0.708 \pm 0.013$ & $0.544 \pm 0.028$ & $0.639 \pm 0.006$ & $0.693 \pm 0.009$\\
	& & \multicolumn{1}{|c|}{100} & $0.630 \pm 0.004$ & $0.736 \pm 0.006$ & $0.589 \pm 0.005$ & $0.662 \pm 0.003$ & $0.723 \pm 0.007$\\
	& &  \multicolumn{1}{|c|}{150} & $0.639 \pm 0.004$ & $0.747 \pm 0.004$ & $0.609 \pm 0.007$ & $0.683 \pm 0.006$ & $0.741 \pm 0.007$\\
	\midrule
	\multirow{5}{*}{\TheName} & \multirow{5}{*}{RFF} & \multicolumn{1}{|c|}{10} & $0.560 \pm 0.004$ & $0.646 \pm 0.007$ & $0.528 \pm 0.005$ & $0.559 \pm 0.004$ & $0.620 \pm 0.004$ \\
	& & \multicolumn{1}{|c|}{20} & $0.572 \pm 0.008$ & $0.656 \pm 0.007$ & $0.535 \pm 0.010$ & $0.572 \pm 0.006$ & $0.633 \pm 0.005$ \\
	& & \multicolumn{1}{|c|}{50} & $0.609 \pm 0.017$ & $0.703 \pm 0.023$ & $0.569 \pm 0.024$ & $0.620 \pm 0.022$ & $0.688 \pm 0.022$ \\
	& & \multicolumn{1}{|c|}{100} & $0.637 \pm 0.009$ & $0.743 \pm 0.007$ & $0.610 \pm 0.012$ & $0.668 \pm 0.006$ & $0.736 \pm 0.009$ \\
	& &  \multicolumn{1}{|c|}{150} & $0.642 \pm 0.006$ & $0.750 \pm 0.005$ & $0.622 \pm 0.008$ & $\underline{0.686 \pm 0.005}$ & $\mathbf{0.747 \pm 0.007}$ \\
	\midrule
	\multicolumn{1}{c|}{} & \multicolumn{2}{c|}{Cosine} & $\mathbf{0.656 \pm 0.003}$ & $0.720 \pm 0.009$ & $\mathbf{0.631 \pm 0.004}$ & $\mathbf{0.696 \pm 0.004}$ & $0.694 \pm 0.005$\\
	\multicolumn{1}{c|}{KPCA} & \multicolumn{2}{c|}{RBF} & $0.641 \pm 0.003$ & $0.709 \pm 0.005$ & $0.605 \pm 0.003$ & $0.666 \pm 0.004$ & $0.695 \pm 0.004$\\
	\multicolumn{1}{c|}{} & \multicolumn{2}{c|}{Sigmoid} & $\underline{0.646 \pm 0.003}$ & $0.729 \pm 0.003$ & $0.599 \pm 0.003$ & $0.678 \pm 0.005$ & $0.719 \pm 0.003$\\
	\midrule
	\multicolumn{1}{c|}{} & \multicolumn{2}{c|}{Cosine} & $\mathbf{0.656 \pm 0.003}$ & $\mathbf{0.756 \pm 0.003}$ & $\mathbf{0.631 \pm 0.004}$ & $\mathbf{0.696 \pm 0.003}$ & $0.735 \pm 0.004$\\
	\multicolumn{1}{c|}{SKPCA} & \multicolumn{2}{c|}{RBF} & $0.640 \pm 0.003$ & $0.709 \pm 0.005$ & $0.605 \pm 0.003$ & $0.664 \pm 0.005$ & $0.695 \pm 0.004$\\
	\multicolumn{1}{c|}{} & \multicolumn{2}{c|}{Sigmoid} & $0.645 \pm 0.003$ & $0.729 \pm 0.004$ & $0.598 \pm 0.003$ & $0.677 \pm 0.004$ & $0.718 \pm 0.004$\\
	\bottomrule
	\end{tabular}
\end{table}

\begin{table}
    \centering
    \footnotesize
    \caption{Prediction Accuracy of Additional Baseline Methods on the Higgs Testing Dataset (values at the \textbf{first places} and the \underline{second places})}\label{tbl:test_acc Higgs baseline} 
	\begin{tabular}{c|c|ccccc}
	\toprule
 	\multicolumn{2}{c}{Method\slash Model} & SVC & RF & Log & MLPC & XGB \\
	\midrule
	\multicolumn{2}{c|}{Original Variables} & $0.655 \pm 0.005$ & $0.775 \pm 0.003$ & $0.639 \pm 0.005$ & $0.682 \pm 0.005$ & $\underline{0.761 \pm 0.003}$\\
	\midrule
	\multirow{3}{*}{Unsupervised} & PCA & $0.661 \pm 0.003$ & $\underline{0.777 \pm 0.002}$ & $0.639 \pm 0.005$ & $0.695 \pm 0.004$ & $\underline{0.761 \pm 0.003}$\\
	& ICA & $\mathbf{0.703 \pm 0.004}$ & $0.769 \pm 0.004$ & $\underline{0.640 \pm 0.004}$ & $\underline{0.704 \pm 0.003}$ & $0.749 \pm 0.003$\\
	& SVD & $0.543 \pm 0.005$ & $0.631 \pm 0.004$ & $0.540 \pm 0.005$ & $0.544 \pm 0.008$ & $0.567 \pm 0.004$ \\
	\midrule
	\multirow{2}{*}{Supervised} & RFE-poly & $0.671 \pm 0.004$ & $\underline{0.777 \pm 0.003}$ & $\mathbf{0.670 \pm 0.005}$ & $0.698 \pm 0.007$ & $\mathbf{0.764 \pm 0.003}$ \\
	& RFE-var & $\underline{0.702 \pm 0.006}$ & $\mathbf{0.779 \pm 0.004}$ & $0.639 \pm 0.004$ & $\mathbf{0.709 \pm 0.006}$ & $0.755 \pm 0.004$\\
	\bottomrule
	\end{tabular}
\end{table}

\begin{table}
    \centering
    \footnotesize
    \caption{Comparison of Prediction MSE Between \TheName{}, KPCA, and SKPCA on the Bias Testing Dataset ($\times 10^2$) (values at the \textbf{first places} and the \underline{second places})}\label{tbl:test_mse Bias} 
	\begin{tabular}{c|cc|cccc}
	\toprule
 	method & kernel & (\# terms) & SVR & RF & Ridge & XGB \\
	\midrule
	\multirow{5}{*}{\TheName} & \multirow{5}{*}{cosine kernel} & \multicolumn{1}{|c|}{10} & $\underline{7.954 \pm 1.352}$ & $0.974 \pm 0.170$ & $3.518 \pm 0.638$ & $1.001 \pm 0.188$\\
	& & \multicolumn{1}{|c|}{20} & $8.077 \pm 1.252$ & $0.718 \pm 0.114$ & $3.035 \pm 0.349$ & $0.682 \pm 0.122$\\
	& & \multicolumn{1}{|c|}{50} & $8.163 \pm 1.180$ & $0.580 \pm 0.067$ & $2.308 \pm 0.208$ & $0.485 \pm 0.058$\\
	& & \multicolumn{1}{|c|}{100} & $8.188 \pm 1.158$ & $0.587 \pm 0.063$ & $2.059 \pm 0.105$ & $0.459 \pm 0.069$ \\
	& & \multicolumn{1}{|c|}{150} & $8.197 \pm 1.152$ & $0.581 \pm 0.060$ & $1.926 \pm 0.278$ & $0.450 \pm 0.066$\\
	\midrule
	\multirow{5}{*}{\TheName} & \multirow{5}{*}{RBF kernel} & \multicolumn{1}{|c|}{10} & $9.352 \pm 0.296$ & $0.725 \pm 0.069$ & $2.826 \pm 0.245$ & $0.742 \pm 0.067$\\
	& & \multicolumn{1}{|c|}{20} & $9.384 \pm 0.297$ & $0.582 \pm 0.062$ & $2.545 \pm 0.171$ & $0.552 \pm 0.048$\\
	& & \multicolumn{1}{|c|}{50} & $9.401 \pm 0.297$ & $\mathbf{0.545 \pm 0.059}$ & $2.258 \pm 0.149$ & $0.461 \pm 0.074$\\
	& & \multicolumn{1}{|c|}{100} & $9.407 \pm 0.298$ & $0.564 \pm 0.069$ & $2.033 \pm 0.131$ & $0.450 \pm 0.055$\\
	& & \multicolumn{1}{|c|}{150} & $9.409 \pm 0.298$ & $0.571 \pm 0.061$ & $\underline{1.847 \pm 0.122}$ & $\mathbf{0.445 \pm 0.060}$\\
	\midrule
	\multirow{5}{*}{\TheName} & \multirow{5}{*}{sigmoid kernel} & \multicolumn{1}{|c|}{10} & $9.377 \pm 0.298$ & $0.761 \pm 0.047$ & $2.658 \pm 0.127$ & $0.668 \pm 0.072$\\
	& & \multicolumn{1}{|c|}{20} & $9.405 \pm 0.300$ & $0.684 \pm 0.057$ & $2.469 \pm 0.141$ & $0.535 \pm 0.062$\\
	& & \multicolumn{1}{|c|}{50} & $9.418 \pm 0.300$ & $0.648 \pm 0.057$ & $2.262 \pm 0.125$ & $0.476 \pm 0.067$ \\
	& & \multicolumn{1}{|c|}{100} & $9.423 \pm 0.301$ & $0.586 \pm 0.052$ & $1.975 \pm 0.105$ & $0.454 \pm 0.069$\\
	& &  \multicolumn{1}{|c|}{150} & $9.411 \pm 0.299$ & $0.582 \pm 0.054$ & $\mathbf{1.806 \pm 0.126}$ & $0.451 \pm 0.077$ \\
	\midrule
	\multirow{5}{*}{\TheName} & \multirow{5}{*}{RFF} & \multicolumn{1}{|c|}{10}  & $8.260 \pm 1.522$ & $0.732 \pm 0.113$ & $3.185 \pm 0.664$ & $0.758 \pm 0.115$ \\
	& & \multicolumn{1}{|c|}{20} & $8.229 \pm 1.605$ & $0.607 \pm 0.100$ & $2.660 \pm 0.279$ & $0.571 \pm 0.108$ \\
	& & \multicolumn{1}{|c|}{50} & $8.329 \pm 1.497$ & $\underline{0.557 \pm 0.099}$ & $2.238 \pm 0.118$ & $0.465 \pm 0.070$ \\
	& & \multicolumn{1}{|c|}{100} & $8.344 \pm 1.478$ & $0.576 \pm 0.064$ & $2.018 \pm 0.111$ & $\underline{0.446 \pm 0.067}$ \\
	& &  \multicolumn{1}{|c|}{150} & $8.353 \pm 1.469$ & $0.575 \pm 0.060$ & $1.898 \pm 0.190$ & $0.450 \pm 0.072$ \\
	\midrule
	\multicolumn{1}{c|}{} & \multicolumn{2}{c|}{cosine kernel} & $\mathbf{4.913 \pm 0.152}$ & $0.697 \pm 0.073$ & $9.276 \pm 0.273$ & $0.757 \pm 0.077$\\
	\multicolumn{1}{c|}{KPCA} & \multicolumn{2}{c|}{RBF kernel} & $9.700 \pm 0.307$ & $9.781 \pm 1.584$ & $9.590 \pm 0.295$ & $9.323 \pm 3.173$\\
	\multicolumn{1}{c|}{} & \multicolumn{2}{c|}{sigmoid kernel} & $9.721 \pm 0.304$ & $9.696 \pm 0.300$ & $9.696 \pm 0.300$ & $9.696 \pm 0.300$\\
	\midrule
	\multicolumn{1}{c|}{} & \multicolumn{2}{c|}{cosine kernel} & $9.721 \pm 0.304$ & $9.696 \pm 0.300$ & $9.696 \pm 0.300$ & $9.696 \pm 0.300$\\
	\multicolumn{1}{c|}{SKPCA} & \multicolumn{2}{c|}{RBF kernel} & $9.700 \pm 0.307$ & $9.734 \pm 1.546$ & $9.590 \pm 0.295$ & $9.314 \pm 2.461$\\
	\multicolumn{1}{c|}{} & \multicolumn{2}{c|}{sigmoid kernel} & $9.721 \pm 0.304$ & $9.696 \pm 0.300$ & $9.696 \pm 0.300$ & $9.696 \pm 0.300$\\
	\bottomrule
	\end{tabular}
\end{table}

\begin{table}
    \centering
    \footnotesize
    \caption{Prediction MSE of Additional Baseline Methods on the Bias Testing Dataset ($\times 10^2$) (values at the \textbf{first places} and the \underline{second places})}\label{tbl:test_mse Bias baseline} 
	\begin{tabular}{c|c|cccc}
	\toprule
 	\multicolumn{2}{c}{Method\slash Model} & SVR & RF & Ridge & XGB \\
	\midrule
	\multicolumn{2}{c|}{Original Variables} &$9.661 \pm 0.304$ & $0.583 \pm 0.056$ & $2.367 \pm 0.143$ & $0.461 \pm 0.065$ \\
	\midrule
	\multirow{3}{*}{Unsupervised} & PCA & $7.617 \pm 0.191$ & $0.584 \pm 0.056$ & $2.367 \pm 0.143$ & $0.460 \pm 0.062$\\
	& ICA & $\mathbf{2.797 \pm 0.137}$ & $0.748 \pm 0.065$ & $2.367 \pm 0.144$ & $0.589 \pm 0.045$ \\
	& SVD & $9.596 \pm 0.301$ & $3.864 \pm 0.153$ & $9.425 \pm 0.249$ & $5.504 \pm 0.182$ \\
	\midrule
	\multirow{4}{*}{Supervised} & RFE-poly  &  $\underline{6.705 \pm 0.592}$ & $0.584 \pm 0.064$ & $\underline{1.869 \pm 0.106}$ & $0.461 \pm 0.060$ \\
	& RFE-var & $\underline{6.705 \pm 0.592}$ & $0.584 \pm 0.064$ & $\underline{1.869 \pm 0.106}$ & $0.461 \pm 0.060$\\
	& LASSO &  $9.407 \pm 0.299$ & $\mathbf{0.570 \pm 0.059}$ & $\mathbf{1.645 \pm 0.095}$ & $\underline{0.440 \pm 0.076}$\\
	& SPORE-LASSO & $8.830 \pm 0.275$ & $\underline{0.575 \pm 0.063}$ & $2.297 \pm 2.724$ & $\mathbf{0.418 \pm 0.069}$ \\
	\bottomrule
	\end{tabular}
\end{table}

\begin{table}
    \centering
    \footnotesize
    \caption{Comparison of Prediction MSE Between \TheName{}, KPCA, and SKPCA on the Superconductivity Testing Dataset ($\times 10^5$) (values at the \textbf{first places} and the \underline{second places})}\label{tbl:test_mse_superconductivity} 
	\begin{tabular}{c|cc|cccc}
	\toprule
 	method & kernel & (\# terms) & SVR & RF & Ridge & XGB \\
	\midrule
	\multirow{5}{*}{\TheName} & \multirow{5}{*}{cosine kernel} & \multicolumn{1}{|c|}{10} & $32.443 \pm 3.172$ & $6.266 \pm 0.430$ & $22.110 \pm 2.144$ & $6.891 \pm 0.460$ \\
	& & \multicolumn{1}{|c|}{20} & $32.588 \pm 3.380$ & $5.700 \pm 0.382$ & $18.679 \pm 1.165$ & $6.165 \pm 0.527$ \\
	& & \multicolumn{1}{|c|}{50} & $32.419 \pm 3.208$ & $5.579 \pm 0.484$ & $15.676 \pm 0.933$ & $5.964 \pm 0.508$ \\
	& & \multicolumn{1}{|c|}{100} & $32.218 \pm 2.989$ & $5.522 \pm 0.441$ & $14.275 \pm 0.737$ & $5.889 \pm 0.527$ \\
	& & \multicolumn{1}{|c|}{150} & $32.234 \pm 2.979$ & $\mathbf{5.453 \pm 0.477}$ & $\underline{13.417 \pm 0.806}$ & $\underline{5.759 \pm 0.423}$ \\
	\midrule
	\multirow{5}{*}{\TheName} & \multirow{5}{*}{RBF kernel} & \multicolumn{1}{|c|}{10} & $34.346 \pm 5.178$ & $6.325 \pm 0.409$ & $21.323 \pm 1.895$ & $6.860 \pm 0.491$ \\
	& & \multicolumn{1}{|c|}{20} & $33.136 \pm 3.147$ & $5.908 \pm 0.495$ & $18.099 \pm 0.970$ & $6.232 \pm 0.388$ \\
	& & \multicolumn{1}{|c|}{50} & $33.018 \pm 2.777$ & $5.648 \pm 0.441$ & $15.684 \pm 0.903$ & $5.980 \pm 0.425$ \\
	& & \multicolumn{1}{|c|}{100} & $33.218 \pm 2.691$ & $5.589 \pm 0.443$ & $14.250 \pm 0.894$ & $5.878 \pm 0.440$ \\
	& & \multicolumn{1}{|c|}{150} & $33.359 \pm 2.635$ & $5.476 \pm 0.414$ & $13.427 \pm 0.794$ & $5.787 \pm 0.429$ \\
	\midrule
	\multirow{5}{*}{\TheName} & \multirow{5}{*}{sigmoid kernel} & \multicolumn{1}{|c|}{10} & $50.489 \pm 0.921$ & $6.716 \pm 0.527$ & $27.152 \pm 0.796$ & $7.358 \pm 0.475$ \\
	& & \multicolumn{1}{|c|}{20} & $50.496 \pm 0.906$ & $5.842 \pm 0.423$ & $20.680 \pm 0.578$ & $6.270 \pm 0.416$ \\
	& & \multicolumn{1}{|c|}{50} & $48.236 \pm 0.878$ & $5.579 \pm 0.441$ & $15.979 \pm 0.565$ & $5.908 \pm 0.582$ \\
	& & \multicolumn{1}{|c|}{100} & $45.160 \pm 0.862$ & $5.549 \pm 0.471$ & $14.476 \pm 0.869$ & $5.885 \pm 0.603$ \\
	& &  \multicolumn{1}{|c|}{150} & $44.677 \pm 0.871$ & $\underline{5.456 \pm 0.453}$ & $13.787 \pm 0.792$ & $\mathbf{5.742 \pm 0.598}$ \\
	\midrule
	\multirow{5}{*}{\TheName} & \multirow{5}{*}{RFF} & \multicolumn{1}{|c|}{10} & $33.648 \pm 4.504$ & $6.229 \pm 0.473$ & $20.588 \pm 0.970$ & $6.784 \pm 0.516$ \\
	& & \multicolumn{1}{|c|}{20} & $33.042 \pm 3.952$ & $5.811 \pm 0.428$ & $17.769 \pm 0.800$ & $6.278 \pm 0.446$ \\
	& & \multicolumn{1}{|c|}{50} & $32.690 \pm 3.156$ & $5.635 \pm 0.436$ & $15.453 \pm 0.843$ & $6.002 \pm 0.436$ \\
	& & \multicolumn{1}{|c|}{100} & $32.679 \pm 3.171$ & $5.556 \pm 0.440$ & $14.279 \pm 0.884$ & $5.915 \pm 0.443$ \\
	& &  \multicolumn{1}{|c|}{150} & $32.789 \pm 3.132$ & $5.481 \pm 0.464$ & $\mathbf{13.392 \pm 0.927}$ & $5.824 \pm 0.548$ \\
	\midrule
	\multicolumn{1}{c|}{} & \multicolumn{2}{c|}{cosine kernel} & $\underline{29.217 \pm 0.674}$ & $6.262 \pm 0.567$ & $21.647 \pm 0.808$ & $6.664 \pm 0.576$ \\
	\multicolumn{1}{c|}{KPCA} & \multicolumn{2}{c|}{RBF kernel} & $57.284 \pm 0.785$ & $67.864 \pm 7.021$ & $53.134 \pm 0.847$ & $52.734 \pm 7.079$ \\
	\multicolumn{1}{c|}{} & \multicolumn{2}{c|}{sigmoid kernel} & $61.746 \pm 0.862$ & $61.312 \pm 0.775$ & $61.311 \pm 0.775$ & $61.311 \pm 0.775$ \\
	\midrule
	\multicolumn{1}{c|}{} & \multicolumn{2}{c|}{cosine kernel} & $\mathbf{29.188 \pm 0.672}$ & $6.411 \pm 0.572$ & $22.936 \pm 0.861$ & $6.941 \pm 0.606$ \\
	\multicolumn{1}{c|}{SKPCA} & \multicolumn{2}{c|}{RBF kernel} & $57.284 \pm 0.785$ & $67.719 \pm 6.926$ & $53.136 \pm 0.847$ & $50.245 \pm 6.701$ \\
	\multicolumn{1}{c|}{} & \multicolumn{2}{c|}{sigmoid kernel} & $61.746 \pm 0.862$ & $61.312 \pm 0.775$ & $61.311 \pm 0.775$ & $61.311 \pm 0.775$ \\
	\bottomrule
	\end{tabular}
\end{table}

\begin{table}
    \centering
    \footnotesize
    \caption{Prediction MSE of Additional Baseline Methods on the Superconductivity Testing Dataset ($\times 10^5$) (values at the \textbf{first places} and the \underline{second places})}\label{tbl:test_mse_superconductivity_baseline} 
	\begin{tabular}{c|c|cccc}
	\toprule
 	\multicolumn{2}{c}{Method\slash Model} & SVR & RF & Ridge & XGB \\
	\midrule
	\multicolumn{2}{c|}{Original Variables} & $34.080 \pm 0.787$ & $\mathbf{5.399 \pm 0.421}$ & $17.240 \pm 0.781$ & $\underline{5.697 \pm 0.473}$ \\
	\midrule
	\multirow{3}{*}{Unsupervised} & PCA & $\mathbf{28.665 \pm 0.695}$ & $5.739 \pm 0.439$ & $17.240 \pm 0.781$ & $5.902 \pm 0.405$ \\
	& ICA & $43.999 \pm 0.777$ & $6.913 \pm 0.613$ & $17.240 \pm 0.781$ & $7.222 \pm 0.616$ \\
	& SVD & $33.852 \pm 0.757$ & $11.788 \pm 0.822$ & $29.472 \pm 0.977$ & $14.202 \pm 0.904$ \\
	\midrule
	\multirow{4}{*}{Supervised} & RFE-poly  & $\underline{30.564 \pm 2.132}$ & $5.454 \pm 0.398$ & $\underline{12.948 \pm 0.855}$ & $5.699 \pm 0.455$ \\
	& RFE-var & $\underline{30.564 \pm 2.132}$ & $5.454 \pm 0.398$ & $\underline{12.948 \pm 0.855}$ & $5.699 \pm 0.455$ \\
	& LASSO & $32.486 \pm 0.796$ & $\underline{5.402 \pm 0.431}$ & $\mathbf{12.044 \pm 0.932}$ & $\mathbf{5.666 \pm 0.425}$ \\
	\bottomrule
	\end{tabular}
\end{table}

The results demonstrate that \TheName{} consistently outperforms KPCA and SKPCA in prediction tasks. On the Yahoo dataset (Tables \ref{tbl:test_acc Yahoo} and \ref{tbl:test_acc Yahoo baseline}), \TheName{} achieves a peak accuracy of $0.531 \pm 0.012$ using the sigmoid kernel with 100 selected features, surpassing KPCA ($0.521 \pm 0.013$) and SKPCA ($0.516 \pm 0.012$). These superior results are consistent across nearly all kernels and datasets, including both regression and classification tasks, highlighting \TheName{} as an effective alternative to kernel-based principal components.

When compared to baseline methods such as PCA, ICA, and Recursive Feature Elimination (RFE)~\cite{RFE}, \TheName{} consistently delivers comparable or better performance across diverse datasets. Specifically, on the Yahoo dataset, \TheName{} outperforms unsupervised methods like PCA and ICA across all prediction models, and even surpasses the supervised method RFE without utilizing label information.

On the Bias dataset (Tables~\ref{tbl:test_mse Bias} and~\ref{tbl:test_mse Bias baseline}), \TheName{} shows superior Mean Squared Error (MSE) performance compared to Lasso and SPORE-Lasso~\cite{SPORE_lasso} using both SVC and RF. For example, with 20 selected features and the RFF kernel, \TheName{} achieves an MSE of $8.229 \pm 1.605$ using RF, outperforming Lasso ($9.407 \pm 0.299$) and SPORE-Lasso ($8.830 \pm 0.275$). This demonstrates that our selective inference approach outperforms traditional sparse estimation methods like Lasso.

The results on the Superconductivity dataset further validate the effectiveness of \TheName{} in feature selection. As shown in Tables \ref{tbl:test_mse_superconductivity} and \ref{tbl:test_mse_superconductivity_baseline}, \TheName{} consistently outperforms KPCA and SKPCA across all kernels and models, achieving competitive or superior performance compared to baseline methods such as PCA, ICA, RFE, and LASSO. Notably, \TheName{} achieves these results without relying on label information, showcasing its versatility and robustness as a feature selection method.

\subsection{Interpretation Analysis of the Higgs Dataset} \label{subsec:Higgs}
This section evaluates how effectively the features selected by \TheName{} capture the underlying data structure. In Section~\ref{subsubsec:higgs_visual}, we visually compare the alignment between KPCs and the selected polynomial features, while Section~\ref{subsubsec:higgs_feature} discusses the significance of the selected features in the Higgs dataset, specifically in the context of its background knowledge.

To provide context, the Higgs dataset simulates experiments to detect the Higgs boson, distinguishing between \textbf{signal events} (Higgs boson decays) and \textbf{background events}~\cite{higgs_dataset}. After excluding categorical variables, the 24 numerical variables were expanded to generate 325 features, including constant terms, linear terms, cross-product terms (defined as the product of two variables), and quadratic terms for selection.

\subsubsection{Visualization of Selected Features}
\label{subsubsec:higgs_visual}
\begin{figure}[htbp]
    \centering
    \subfloat[Top Two KPCs]{
        \includegraphics[width=0.33\textwidth]{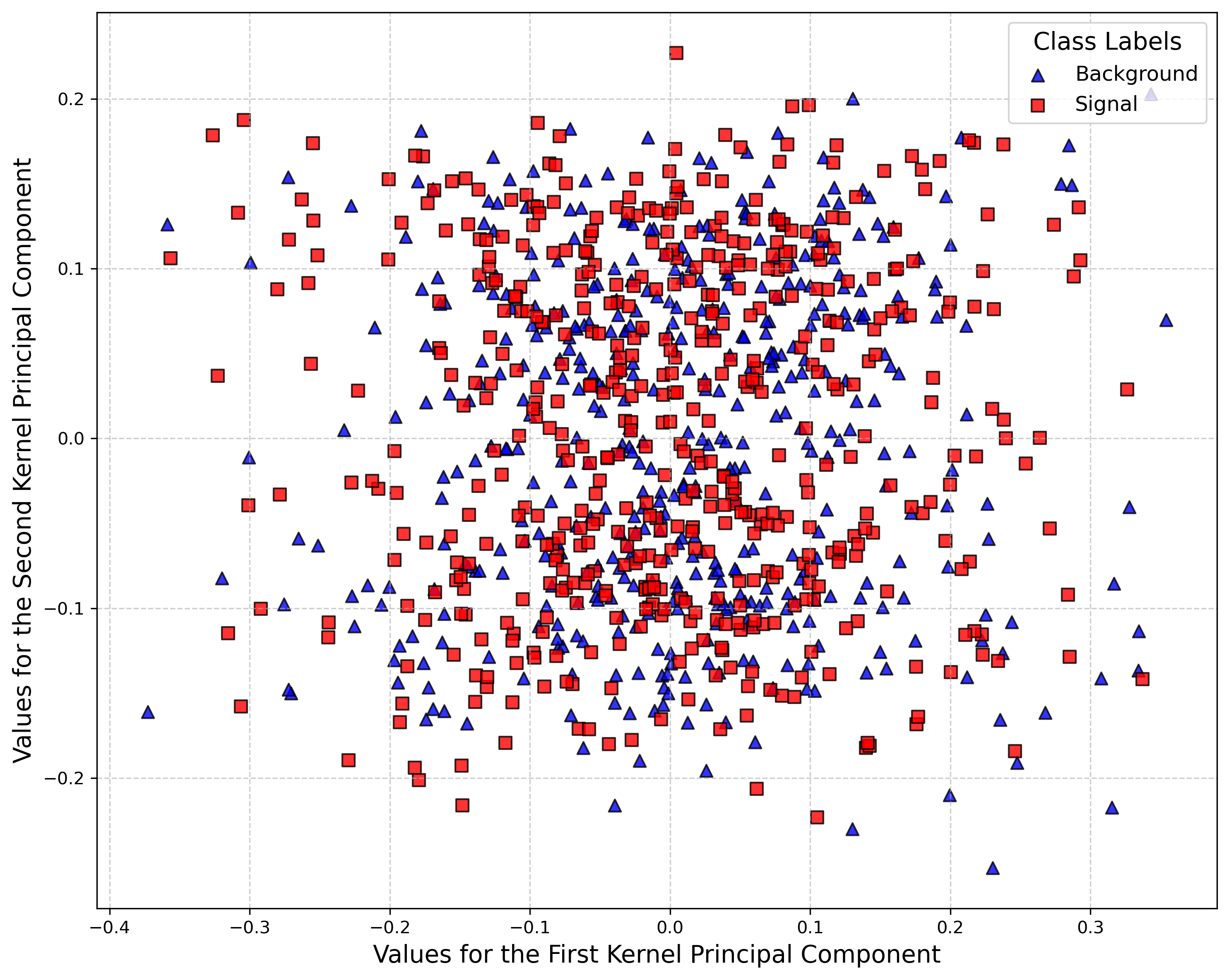}
        \label{fig2a}
    }
    \subfloat[Fitted Top Two KPCs]{
        \includegraphics[width=0.33\textwidth]{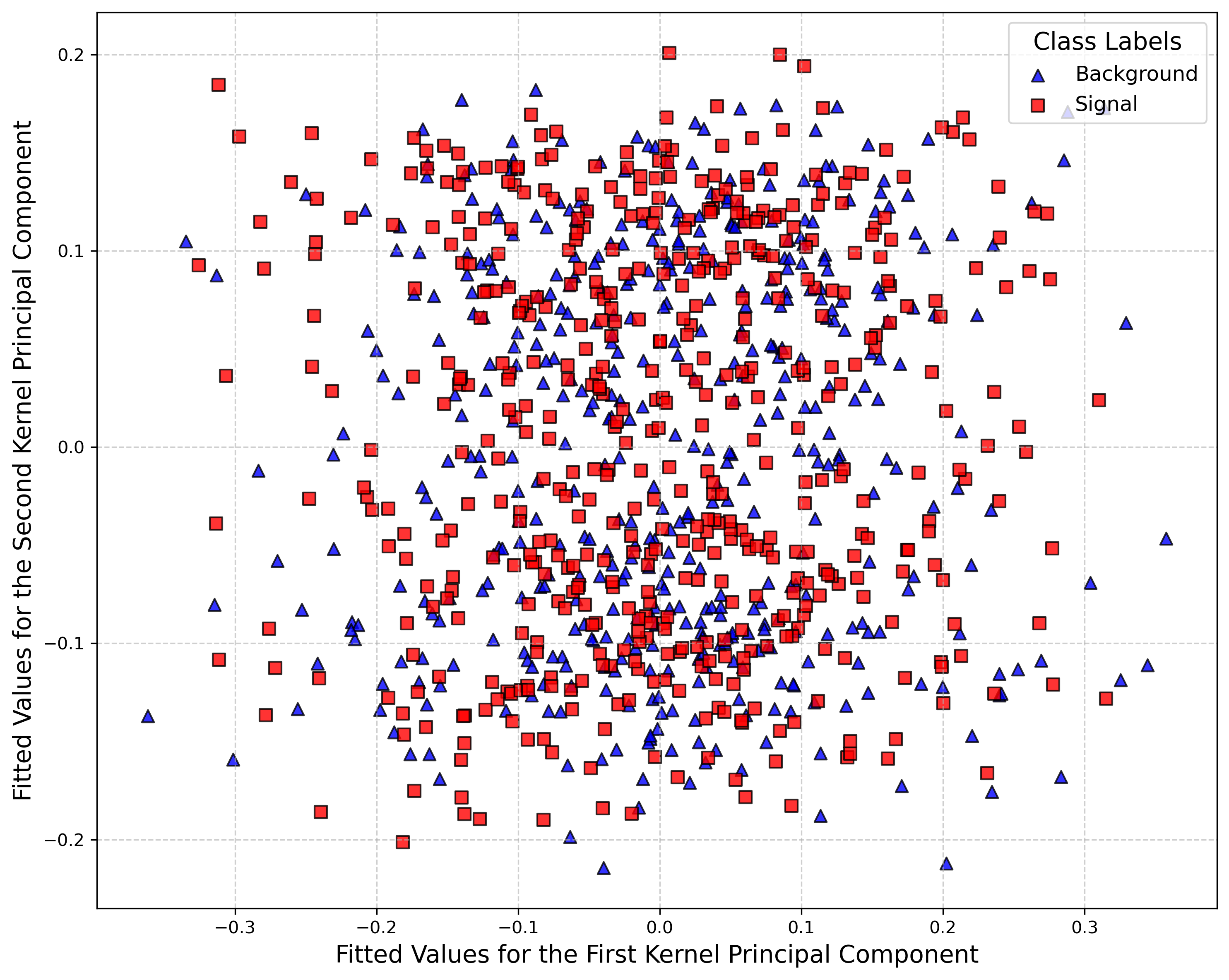}
        \label{fig2b}
    }
    \subfloat[Top 2 Features]{
        \includegraphics[width=0.33\textwidth]{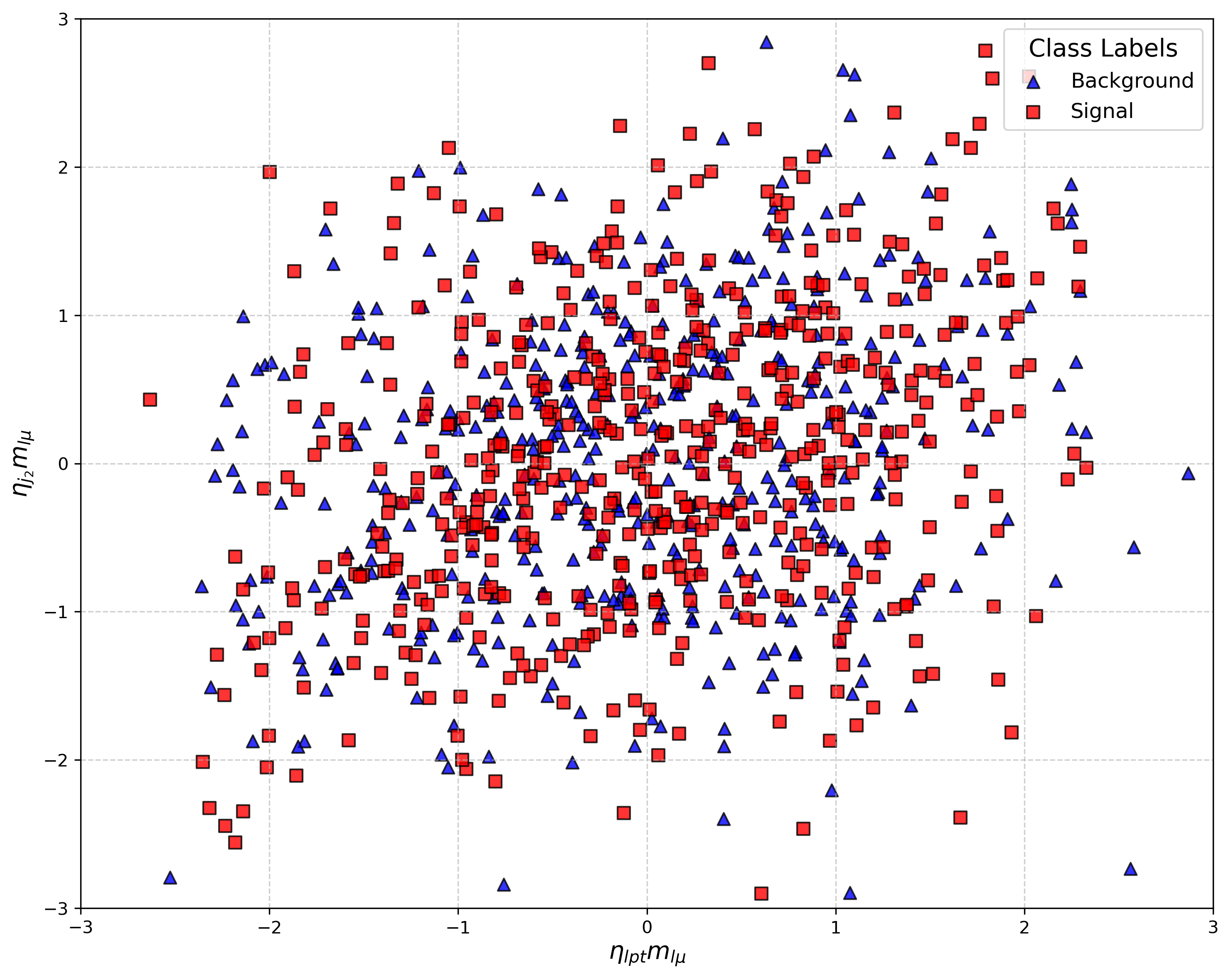}
        \label{fig2c}
    }
    \\
    \subfloat[The 3$^{rd}$ and 4$^{th}$ Features]{
        \includegraphics[width=0.33\textwidth]{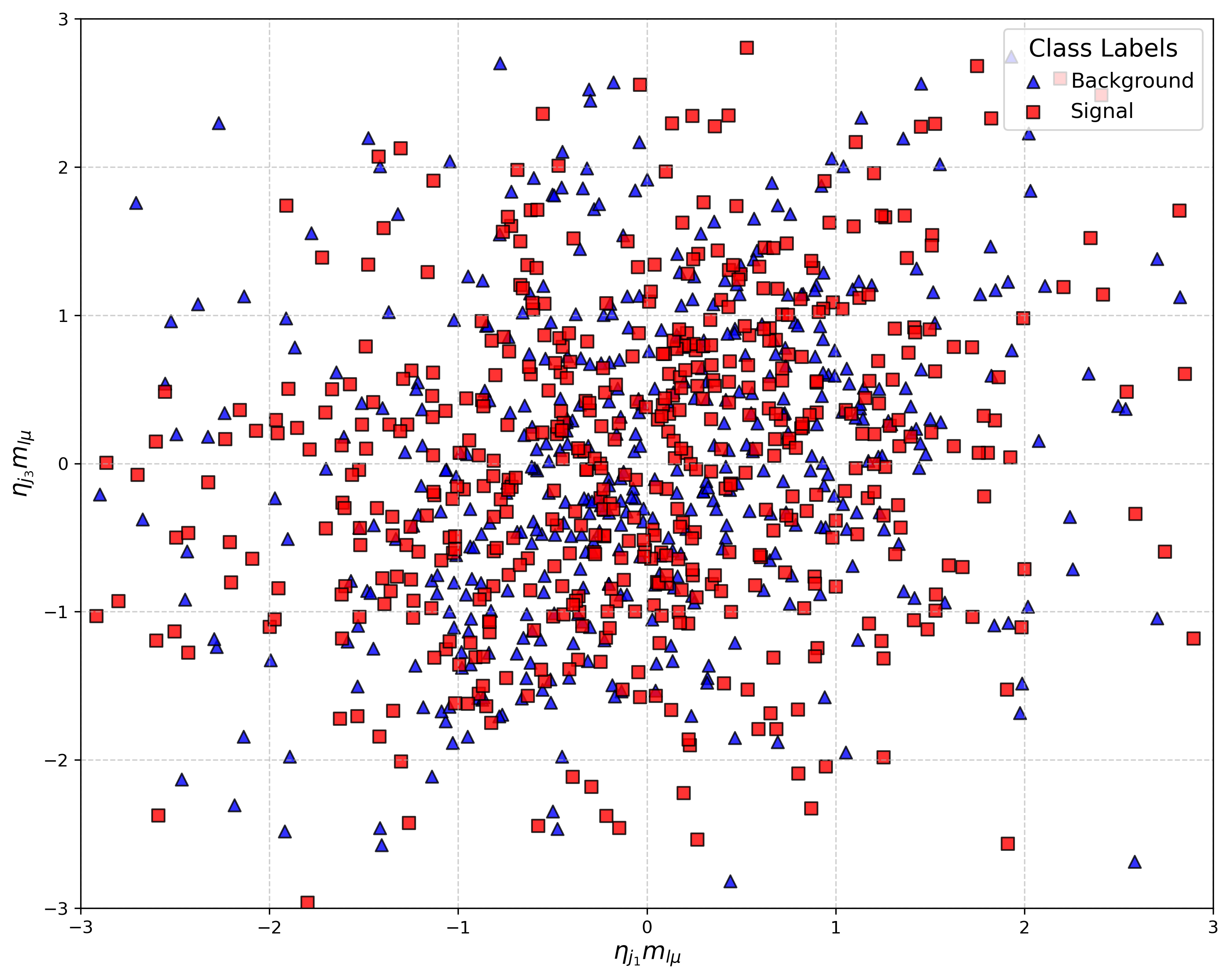}
        \label{fig2d}
    }
    \subfloat[The 4$^{th}$ and 5$^{th}$ Features]{
        \includegraphics[width=0.33\textwidth]{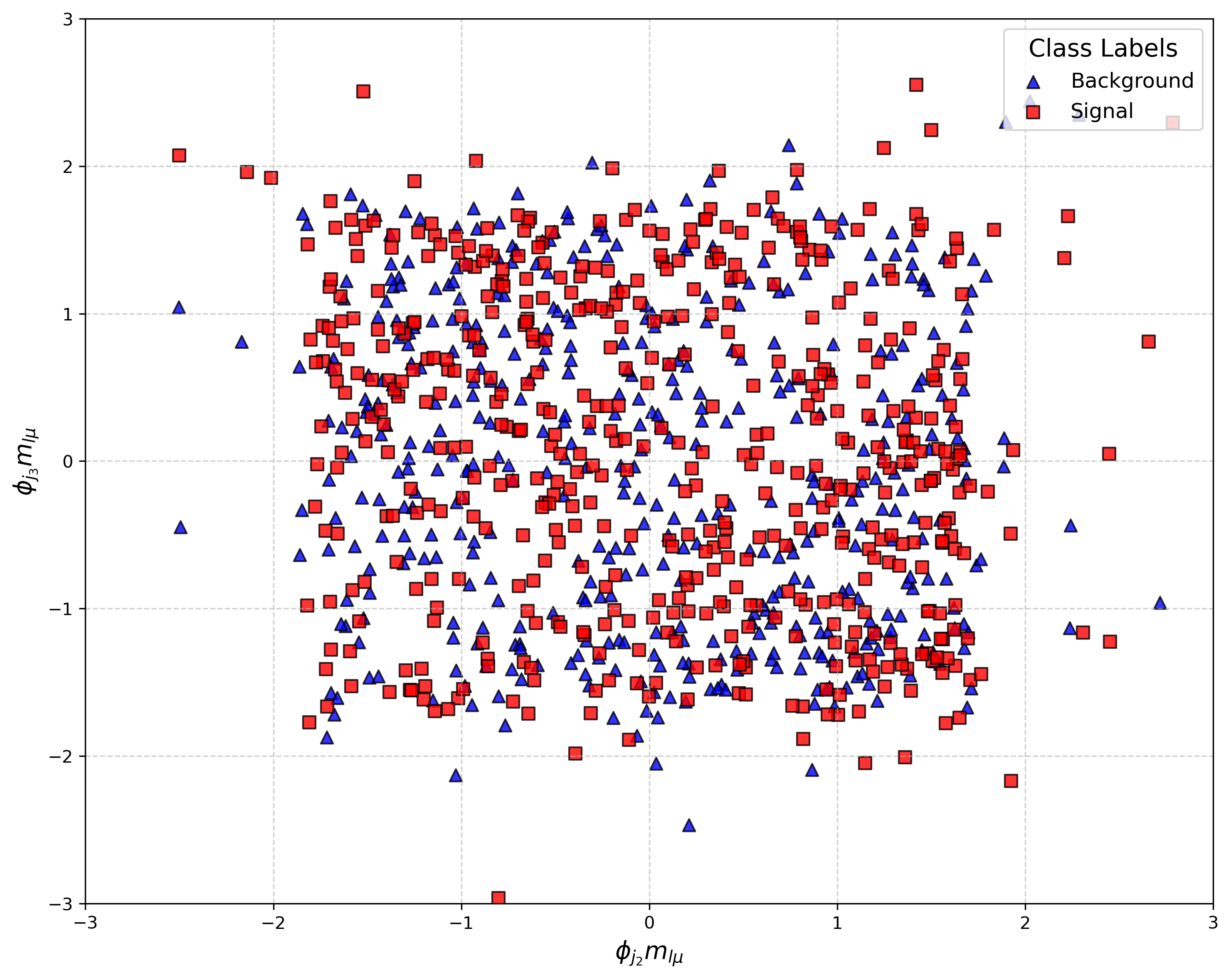}
        \label{fig2e}
    }
    \subfloat[Top Two Linear PCs]{
        \includegraphics[width=0.33\textwidth]{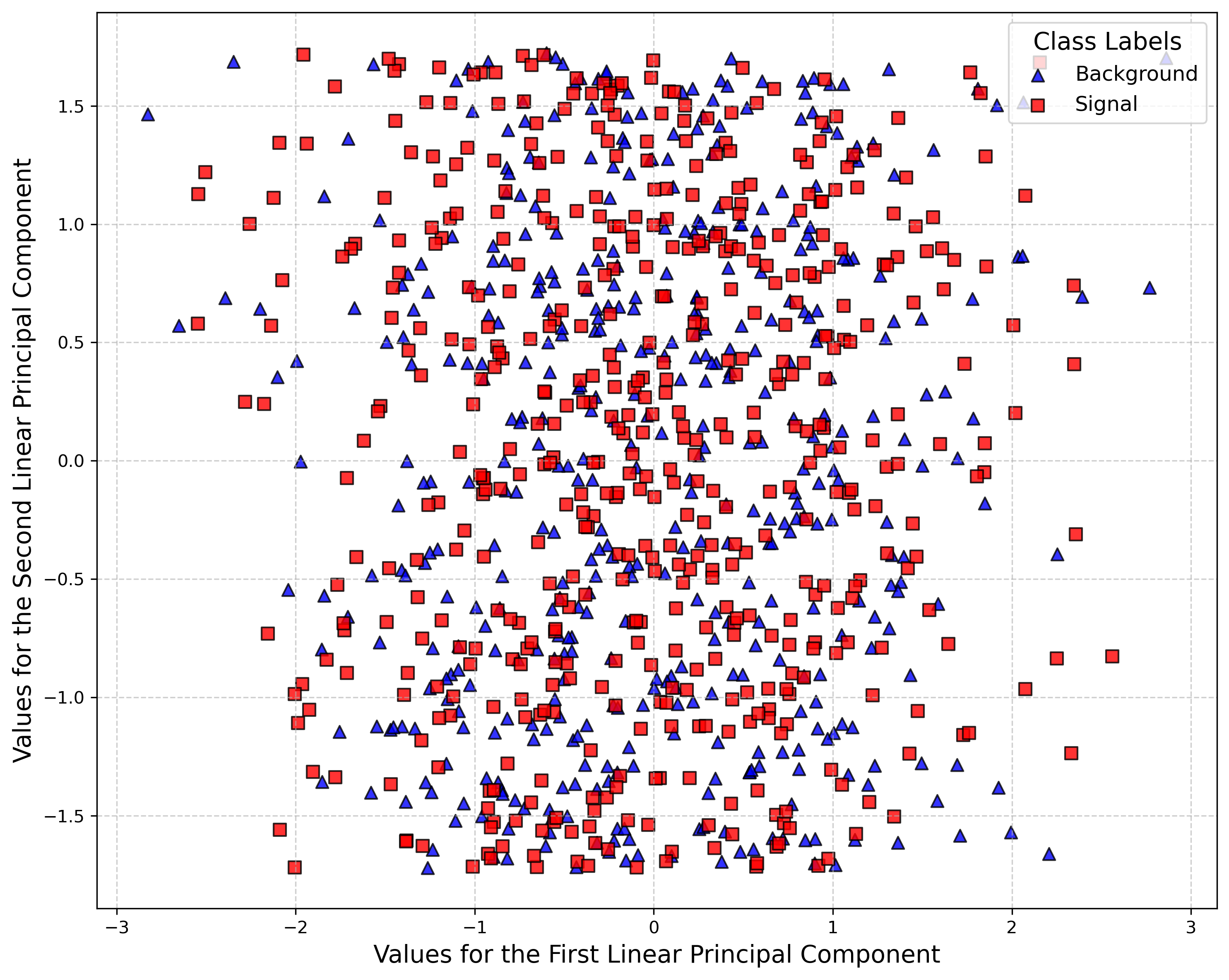}
        \label{fig2f}
    }
    \caption{Comparison of Dimensionality Reduction: Top Polynomial Features Identified by \TheName{} Versus Top Two Principal Components by KPCA and Linear PCA}
    \label{fig:visualization}
\end{figure}

To provide a clear understanding of the feature selection process, figure~\ref{fig:visualization} provides a visualization of the Higgs dataset, comparing the kernel principal components (KPCs) with features selected by \TheName{}, as well as the linear PCA method. Subfigure~\ref{fig2a} shows the first two KPCs from the sigmoid kernel, \ref{fig2b} displays their fitted values using polynomial variables selected by \TheName{}, \ref{fig2c} to \ref{fig2e} present the top six polynomial features selected by \TheName{}, and \ref{fig2f} illustrates the first two linear PCs. Solid circles represent signal samples, while hollow triangles represent background samples.

This visualization reveals several key insights. First, the fitted values in Figure~\ref{fig2b} closely mirror the true KPCs, indicating that \TheName{} effectively captures their structure. Second, as illustrated in Figures~\ref{fig2c} to \ref{fig2e}, the polynomial features selected by \TheName{} display similar distribution patterns to the KPCs, with signal samples dispersed among background samples and a slight rightward shift in the center of the signal samples. Similar findings across other kernel functions and datasets further confirm the effectiveness of \TheName{} in consistently selecting features that capture the key structure of the KPCs.

\subsubsection{Feature Selection Results Analysis}\label{subsubsec:higgs_feature}

\begin{table}
\caption{Interpretations of Selected Polynomial Features for the Higgs Dataset}
\label{tab:feature_interactions_Higgs}
\centering
\footnotesize
\begin{tabular}{p{12cm}|c}
\toprule
\multicolumn{1}{c|}{\textbf{Polynomial Features}/\textbf{Feature Interactions}} & \textbf{p-Value} \\
\hline
$\phi_{j1} \cdot m_{lv}$:\textbf{Azimuthal angle (\(\phi_{j1}\))} of the first jet weighted by the \textbf{invariant mass of lepton-neutrino system (\(m_{lv}\))}, indicating the transverse motion of the jet and its correlation with event energy.  &   0.003\\
\hline
$\eta_{j4} \cdot m_{lv}$: \textbf{Pseudorapidity (\(\eta_{j4}\))} of the fourth jet weighted by the \textbf{invariant mass of lepton-neutrino system (\(m_{lv}\))}, capturing the spatial distribution of the jet along the beam axis and its relationship with event energy. &   0.006\\
\hline
$\eta_{j3} \cdot m_{lv}$: \textbf{Pseudorapidity (\(\eta_{j3}\))} of the third jet weighted by the \textbf{invariant mass of lepton-neutrino system (\(m_{lv}\))}, reflecting the jet motion along the beam axis and its relation to the event energy.  &    0.009\\
\hline
$\eta_{lpt} \cdot m_{lv}$: \textbf{Pseudorapidity (\(\eta_{lpt}\))} of the lepton weighted by the \textbf{invariant mass of lepton-neutrino system (\(m_{lv}\))}, representing the lepton longitudinal motion and its effect on the energy scale. &   0.012 \\
\hline
$\phi_{j4} \cdot m_{lv}$: \textbf{Azimuthal angle (\(\phi_{j4}\))} of the fourth jet weighted by the \textbf{invariant mass of lepton-neutrino system (\(m_{lv}\))}, indicating the relationship between the jet transverse direction and event energy. &   0.015\\
\hline
$p_{t,lpt} \cdot m_{lv}$: \textbf{Transverse momentum (\(p_{t,lpt}\))} of the lepton weighted by the \textbf{invariant mass of lepton-neutrino system (\(m_{lv}\))}, showing the lepton transverse motion and its correlation with event energy. &   0.019\\
\hline
$\phi_{j2} \cdot m_{lv}$: \textbf{Azimuthal angle (\(\phi_{j2}\))} of the second jet weighted by the \textbf{invariant mass of lepton-neutrino system (\(m_{lv}\))}, capturing the jet transverse direction and its relationship with event energy.  &   0.022\\
\hline
$m_{bb} \cdot m_{lv}$: \textbf{\(b\)-jet pair mass (\(m_{bb}\))} weighted by the \textbf{invariant mass of lepton-neutrino system (\(m_{lv}\))}, highlighting the presence of \(b\)-quarks in Higgs decay events. &   0.025\\
\hline
$pt_{j1} \cdot m_{lv}$: \textbf{Transverse momentum (\(p_{t,j1}\))} of the first jet weighted by the \textbf{invariant mass of lepton-neutrino system (\(m_{lv}\))}, indicating the jet transverse motion and its contribution to event energy. &   0.028\\
\hline
$m_{jj} \cdot m_{lv}$: \textbf{Jet pair mass (\(m_{jj}\))} weighted by the \textbf{invariant mass of lepton-neutrino system (\(m_{lv}\))}, describing the relationship between jet pair mass and event energy scale. &   0.031 \\
\hline
$\phi_{j1}$: \textbf{Azimuthal angle (\(\phi_{j1}\))} of the first jet, reflecting the jet  transverse motion in relation to the event energy. &   0.034\\
\hline
$\eta_{j4}$: \textbf{Pseudorapidity (\(\eta_{j4}\))} of the fourth jet, capturing the spatial distribution of the jet along the beam axis and its connection to event energy. &   0.037\\
\hline
$\eta_{j3}$: \textbf{Pseudorapidity (\(\eta_{j3}\))} of the third jet, representing the jet  motion along the beam axis and its relation to the event energy. &   0.040\\
\hline
$p_{t,j4}$: \textbf{Transverse momentum (\(p_{t,j4}\))} of the fourth jet, reflecting its transverse motion and contribution to event dynamics. &   0.043\\
\hline
$\phi_{j_3}m_{lv}$: \textbf{Azimuthal angle (\(\phi_{j_3}\))} of the third jet combined with the \textbf{invariant mass of the lepton-neutrino system (\(m_{lv}\))}, highlighting the transverse direction of the jet and its relation to the event energy scale.  &    0.046\\
\hline
$m_{bb}$: \textbf{Invariant mass (\(m_{bb}\))} of a \(b\)-jet pair, critical for identifying Higgs boson decays into \(b\)-quarks.  &      0.049\\
\bottomrule
\end{tabular}
\end{table}

Table~\ref{tab:feature_interactions_Higgs} presents the results of feature selection from one of the representative repeated experiments, where the weighted significance scores are derived from the sigmoid kernel principal components. Features were selected based on \(p\)-values computed from empirical percentiles of the scores, as described in equation (\ref{eq:pvalue}). A selection threshold of 0.05 was applied, and the features with the smallest \(p\)-values were identified as the most significant according to this criterion. The selected set includes features that combine the invariant mass \(m_{l\nu}\) of the lepton-neutrino system with pseudorapidity \(\eta\), azimuthal angle \(\phi\), and transverse momentum \(p_t\) for jets and leptons, along with jet pair mass \(m_{jj}\), \(b\)-jet pair mass \(m_{bb}\), and individual variables such as \(\eta\), \(\phi\), and transverse momenta for jets. These features capture key spatial and energetic relationships, linking particle motion characteristics with the energy scale of the event to effectively distinguish between signal and background. 

Notably, the combination features that integrate these variables with the invariant mass (\(m_{lv}\)) stand out due to their small \(p\)-values, as shown in Table~\ref{tab:feature_interactions_Higgs}. This is particularly meaningful because in the semi-leptonic decay mode, the invariant mass \(m_{l\nu}\), strongly linked to the \(W\) boson mass, plays a key role in identifying intermediate states in Higgs decays \cite{higgs_dataset}. By combining \(m_{l\nu}\) with angular variables (\(\eta\) and \(\phi\)), these features effectively capture the essential physical quantities of the system, facilitating representation of the kinematic and energy characteristics of the event.

For example, the following combination features highlight the significance of these interactions:
\begin{itemize}
    \item \textbf{Lepton $\eta \cdot m_{lv}$}: Reflects how the longitudinal position of the lepton along the beam axis correlates with the energy scale of the event, helping to distinguish differences in particle jet directions between signal and background events.
    \item \textbf{Missing Energy $\phi \cdot m_{lv}$}: Links the azimuthal angle of the missing transverse energy (associated with the neutrino) to the energy scale of the event, highlighting the movement of the neutrino in relation to the event dynamics and aiding in the differentiation of event topologies.
    \item \textbf{Jet $(\eta/\phi) \cdot m_{lv}$}: Encodes the spatial and energetic relationships of jets by combining the ratio of the jet pseudorapidity to its azimuthal angle with the invariant mass, revealing interactions between jets and the lepton-neutrino system.
\end{itemize}

\subsection{Interpretation Analysis of the Superconductivity Dataset} \label{subsec:Superconductivity}  

\begin{table}
\caption{Interpretations of Selected Polynomial Features for the Superconductivity Dataset}
\label{tab:feature_interactions_Super}
\footnotesize
\centering
\begin{tabular}{p{11cm}|c}
\toprule
\multicolumn{1}{c|}{\textbf{Polynomial Features}/\textbf{Feature Interactions}}  & \textbf{p-Value}\\
\hline
\textbf{range\_fie * wtd\_range\_ThermalConductivity}: The product of \textbf{the range of First Ionization Energy} and \textbf{the weighted range of Thermal Conductivity}, capturing the relationship between the variability of ionization energy and thermal conductivity. & 0.003\\
\hline
\textbf{wtd\_mean\_fie$^2$}: The square of \textbf{the weighted mean of First Ionization Energy}, capturing the squared contribution of ionization energy weighted by certain factors. & 0.006\\
\hline
\textbf{std\_fie * wtd\_mean\_fie}: The product of \textbf{the standard deviation of First Ionization Energy} and \textbf{the weighted mean of First Ionization Energy}, highlighting the relationship between ionization energy's variability and its weighted average. & 0.009\\
\hline
\textbf{range\_atomic\_mass * wtd\_mean\_fie}: The product of \textbf{the range of Atomic Mass} and \textbf{the weighted mean of First Ionization Energy}, reflecting how the variability of atomic mass interacts with the weighted average of ionization energy. & 0.011\\
\hline
\textbf{gmean\_ThermalConductivity * wtd\_mean\_fie}: The product of \textbf{the geometric mean of Thermal Conductivity} and \textbf{the weighted mean of First Ionization Energy}, representing the relationship between the central tendency of thermal conductivity and the weighted average of ionization energy. & 0.014\\
\hline
\textbf{range\_fie * wtd\_mean\_fie}: The product of \textbf{the range of First Ionization Energy} and \textbf{the weighted mean of First Ionization Energy}, capturing the relationship between the variability and weighted average of ionization energy. & 0.017\\
\hline
\textbf{wtd\_range\_ElectronAffinity * wtd\_mean\_fie}: The product of \textbf{the weighted range of Electron Affinity} and \textbf{the weighted mean of First Ionization Energy}, highlighting the interaction between the weighted variability of electron affinity and the weighted average of ionization energy. & 0.020\\
\hline
\textbf{range\_fie$^2$}: The square of \textbf{the range of First Ionization Energy}, representing the squared variation in ionization energy. & 0.023\\
\hline
\textbf{wtd\_range\_atomic\_mass * wtd\_mean\_fie}: The product of \textbf{the weighted range of Atomic Mass} and \textbf{the weighted mean of First Ionization Energy}, reflecting the relationship between the weighted variability of atomic mass and the weighted average of ionization energy. & 0.026\\
\hline
\textbf{entropy\_ElectronAffinity * wtd\_mean\_fie}: The product of \textbf{the entropy of Electron Affinity} and \textbf{the weighted mean of First Ionization Energy}, capturing the relationship between the uncertainty in electron affinity and the weighted average of ionization energy. & 0.029\\
\hline
\textbf{std\_ThermalConductivity * wtd\_mean\_fie}: The product of \textbf{the standard deviation of Thermal Conductivity} and \textbf{the weighted mean of First Ionization Energy}, showing how the variability of thermal conductivity relates to the weighted average of ionization energy. & 0.031\\
\hline
\textbf{entropy\_ThermalConductivity * wtd\_mean\_fie}: The product of \textbf{the entropy of Thermal Conductivity} and \textbf{the weighted mean of First Ionization Energy}, reflecting the interaction between the uncertainty of thermal conductivity and the weighted average of ionization energy. & 0.034\\
\hline
\textbf{std\_fie * wtd\_range\_ThermalConductivity}: The product of \textbf{the standard deviation of First Ionization Energy} and \textbf{the weighted range of Thermal Conductivity}, capturing how the variability in ionization energy relates to the weighted variability in thermal conductivity. & 0.037\\
\hline
\textbf{wtd\_range\_atomic\_radius * wtd\_mean\_fie}: The product of \textbf{the weighted range of Atomic Radius} and \textbf{the weighted mean of First Ionization Energy}, reflecting the relationship between the weighted variability in atomic radius and the weighted average of ionization energy. & 0.040\\
\hline
\textbf{wtd\_range\_atomic\_mass * range\_fie}: The product of \textbf{the weighted range of Atomic Mass} and \textbf{the range of First Ionization Energy}, showing how the weighted variability of atomic mass relates to the variability in ionization energy. & 0.043\\
\hline
\textbf{mean\_FusionHeat * wtd\_mean\_fie}: The product of \textbf{the mean of Fusion Heat} and \textbf{the weighted mean of First Ionization Energy}, capturing the relationship between the central tendency of fusion heat and the weighted average of ionization energy. & 0.046\\
\hline
\textbf{mean\_Valence * range\_fie}: The product of \textbf{the mean of Valence} and \textbf{the range of First Ionization Energy}, reflecting the relationship between the central tendency of valence and the variability of ionization energy. & 0.048\\
\bottomrule
\end{tabular}
\end{table}

This section evaluates how effectively the features selected by \TheName{} capture the physical properties of superconductors, with a focus on their relevance to superconductivity physics. The Superconductivity dataset \cite{superconductivity} predicts the critical temperature (\(T_c\)) based on material properties such as Atomic Mass, Ionization Energy, Atomic Radius, and Thermal Conductivity. A total of 81 variables are derived from these properties, including weighted and geometric means (where the weights correspond to the ratios of elements in the material), entropy, range, and standard deviation, aiming to capture the thermal and structural behavior of the material (see Tables 1 and 2 in \cite{superconductivity}). A preliminary selection step using the Knockoff filter identified 25 numerical variables, which were then used to generate 351 features, including constant terms, linear terms, cross-product terms (defined as the product of two variables), and quadratic terms, for further feature selection.

During feature selection, we conducted multiple experiments to identify the most significant features. Table~\ref{tab:feature_interactions_Super} presents the results of a representative single experiment, where the scores are derived from the radial basis function kernel principal components. These features were selected based on their \(p\)-values computed from empirical percentiles of the scores, with a selection threshold set at 0.05, identifying the most statistically significant features in the context of this experiment.

Among these, combination features that integrate weighted averages and ranges stand out, as indicated by their small \(p\)-values in Table~\ref{tab:feature_interactions_Super}. These features reveal complex interactions between material properties.

For instance, the following combination features serve as examples:
\begin{itemize} 
    \item \texttt{range\_fie} $\times$ \texttt{wtd\_range\_ThermalConductivity}: This feature combines fluctuations in ionization energy with the weighted range of thermal conductivity. Physically, this may relate to how variations in the energy required to remove an electron influence the efficiency of phonon-mediated heat transport, which jointly affects electrical and thermal conductivity in materials.
    \item \texttt{wtd\_range\_ElectronAffinity} $\times$ \texttt{wtd\_mean\_fie}: This feature combines variations in electron affinity with the mean ionization energy. From a physical perspective, it captures the relationship between the tendency of an atom to gain or lose electrons and the energy required for ionization, which can influence charge transfer processes and chemical reactivity.
\end{itemize}

A notable observation is the frequent appearance of the weighted mean of First Ionization Energy (\texttt{wtd\_mean\_fie}) in the selected features, which can be attributed to its role in capturing the central tendency of ionization energy across different material compositions. This central tendency significantly influences various other physical properties, such as charge transfer efficiency and redox stability. Beyond these examples, the selection of features involving atomic mass and valence electron properties suggest additional connections between the lattice structure and electronic interactions. The emergence of such interpretable patterns indicates that the feature selection process not only identifies statistically significant predictors but also aligns with fundamental physical principles, providing valuable insights into the mechanisms governing superconductivity.

\subsection{Selective Inference versus Sparse Estimation}
\TheName{} employs a Knockoff-based selective inference approach as an effective alternative to traditional sparse estimation methods like Lasso. Our experiments across regression datasets demonstrate that \TheName{} often achieves lower MSE than both Lasso and SPORE-Lasso (See Table~\ref{tbl:test_mse Bias} and \ref{tbl:test_mse Bias baseline} in Section~\ref{subsec:overallComparison}, as well as the Supplementary materials). This advantage is driven by two key factors: first, it avoids the bias introduced by regularization in sparse estimation, enabling it to capture the underlying data structure more accurately. Second, as a self-supervised method, \TheName{} is more robust in scenarios with limited or noisy label information.

\section{Discussion and Conclusions}
In this work, we introduced \TheName{}--a transformative, self-supervised approach for feature selection that directly addresses the interpretability limitations inherent in KPCA. By integrating self-supervised signal generation, knockoff-based feature selection, and rigorous hypothesis testing into a cohesive framework, \TheName{} overcomes the complexity challenges typically associated with nonlinear methods. This unified approach not only yields a more transparent mapping between the transformed feature space and the original data but also provides statistically sound feature attribution.

Extensive experimental evaluations on diverse datasets demonstrate that \TheName{} effectively reconstructs kernel principal components while often outperforming conventional KPCA and sparse KPCA in both regression and classification tasks. The method’s ability to derive features that are both robust and interpretable makes it particularly well-suited for applications where understanding the underlying data structure is essential. For instance, practitioners in healthcare, finance, or scientific discovery can directly relate the selected polynomial features to meaningful interactions in the data, facilitating a deeper understanding of complex phenomena. Moreover, the statistical rigor embedded in \TheName{}--through the use of multi-objective knockoff selection and significance testing--ensures that the extracted features are not only informative but also reliable under various experimental settings. This aspect is critical when deploying data-driven insights in high-stakes environments, where confidence in model interpretation is as important as predictive performance.

Looking ahead, several promising directions warrant further exploration. Future research could extend the \TheName{} framework by investigating alternative kernel functions or exploring adaptive polynomial bases tailored to specific data domains. Additionally, integrating the approach with advanced visualization tools~\cite{zhang2023adavis} might further enhance its utility and interpretability for data scientists.


\section*{Declaration}
\begin{itemize}
\item Funding -  Not applicable

\item Conflicts of interest/Competing interests - Not applicable

\item Ethics approval - No data have been fabricated or manipulated to support your conclusions. No data, text, or theories by others are presented as if they were our own. Data we used, the data processing and inference phases do not contain any user personal information. This work does not have the potential to be used for policing or the military.

\item Consent to participate - Not applicable

\item Consent for publication - Not applicable

\item Availability of data and material - Experiments are based on publicly available open-source datasets.

\item Code availability - The code will be made available after the paper is submitted.

\item Authors' contributions - H.X contributed the original idea. X.Z conducted experiments. X.Z and H.X wrote the manuscript. X.Z and H.X share equal technical contribution.
\end{itemize} 

\bibliography{ref}

\end{document}